\documentclass[twoside,11pt]{article}

\usepackage{blindtext}

\usepackage{algorithm}
\usepackage{algorithmicx}
\usepackage{algcompatible}
\usepackage{xcolor}
\usepackage{graphicx}
\usepackage{enumitem}
\setlist[itemize]{leftmargin=*, itemsep=3pt, topsep=3pt}
\usepackage{bbm}
\usepackage{mathtools}
\usepackage{multirow}
\usepackage[normalem]{ulem}
\usepackage{booktabs}
\usepackage{tikz}

\usepackage{balance}

\usepackage{url}
\usepackage{pgfplots}
\usepackage{footnote}
\mathtoolsset{showonlyrefs}

\usepackage{silence}
\ActivateWarningFilters[captions]
\usepackage{subcaption}
\DeactivateWarningFilters[captions]

\usepackage[abbrvbib, preprint]{jmlr2e}
\usepackage{hyperref}
\usepackage{url}

\hypersetup{hidelinks}

\usepgfplotslibrary{groupplots}
\usetikzlibrary{arrows.meta,positioning,calc,fit,backgrounds}
\pgfplotsset{compat=1.18}

\newlength{\figwidth}
\newlength{\figheight}

\newcommand{\prob}{\mathbb{P}}
\newcommand{\R}{\mathbb{R}}
\newcommand{\E}{\mathbb{E}}

\newcommand{\X}{\mathbb{X}}
\newcommand{\Y}{\mathbb{Y}}
\newcommand{\D}{\mathcal{D}}
\newcommand{\M}{\mathcal{M}}
\newcommand{\loss}{\mathcal{L}}
\newcommand{\DD}{\mathbb{D}}

\newcommand{\event}{\mathcal{T}}

\newcommand{\Adj}{\text{Adj}}
\newcommand{\Renyi}{\text{Rényi} }

\newcommand{\Simplex}{S^{d-1}}
\newcommand{\softmax}{\text{softmax} }
\DeclareMathOperator{\ps}{PoissonSample}
\newcommand{\onevec}{\mathbbm{1}}

\usepackage{footnote}

\usepackage{lastpage}
\jmlrheading{23}{2026}{1-\pageref{LastPage}}{1/21; Revised 5/22}{9/22}{21-0000}{Huaiyuan Rao, Calvin Hawkins, Alexander Benvenuti, Matthew Hale}
\ShortHeadings{End-to-End Differential Privacy in Training Deep Neural Network Classifiers}{Rao, Hawkins, Benvenuti and Hale}
\firstpageno{1}

\begin{document}

\title{End-to-End Differential Privacy in Training Deep Neural Network Classifiers}

\author{\name Huaiyuan Rao \email hrao43@gatech.edu \\
       \addr School of Electrical and Computer Engineering, \\
       Georgia Institute of Technology\\
       Atlanta, GA 30332, USA
       \AND
       \name Calvin Hawkins \email chawkins64@gatech.edu \\
       \addr School of Electrical and Computer Engineering, \\
       Georgia Institute of Technology\\
       Atlanta, GA 30332, USA
       \AND
       \name Alexander Benvenuti \email abenvenuti3@gatech.edu \\
       \addr School of Electrical and Computer Engineering, \\
       Georgia Institute of Technology\\
       Atlanta, GA 30332, USA
       \AND
       \name Matthew Hale \email matthale@gatech.edu \\
       \addr School of Electrical and Computer Engineering, \\
       Georgia Institute of Technology\\
       Atlanta, GA 30332, USA
       }

\editor{My editor}
\maketitle

\begin{abstract}%   <- trailing '%' for backward compatibility of .sty file
Differentially private machine learning enables model training on sensitive data while ensuring that individual data is unlikely to be recoverable from the parameters of the resulting model.
However, existing work often privatizes both training inputs and their labels, and these protections 
may be conservative when labels are public or can be safely made public.
Therefore, in this work 
we propose a novel private training framework that instead 
privatizes training inputs 
while keeping labels public.
We consider neural networks with softmax output layers, and thus 
the mapping from training inputs to the 
output of the 
softmax layer is a 
mapping onto the unit simplex. 
We randomize softmax outputs during training by applying the Dirichlet mechanism to enforce differential privacy
for the training inputs, hence the ``end-to-end'' label. 
Because training data is reused across multiple training epochs, we use the notion of \Renyi differential privacy to 
formulate tight bounds on the strength of privacy provided by the Dirichlet mechanism across
repeated uses. 
We show empirically that we 
attain new state-of-the-art accuracy when training from scratch on CIFAR10, MNIST, MedMNIST, FashionMNIST, and SVHN across all privacy budgets evaluated.
Notably, 
when implementing~$(\epsilon, \delta)$-differential privacy 
with $\delta=10^{-5}$, 
we improve the prior state-of-the-art accuracy from $78.37\%$ to $88.17\%$ at $\epsilon=4$ on CIFAR10, and our approach
has $82.96\%$ accuracy even for $\epsilon=1$, which significantly outperforms prior work. 
\end{abstract}


\section{Introduction}
Machine learning (ML) models trained on sensitive information such as medical images, biometric data, or financial records are 
vulnerable to privacy attacks. Examples include reconstruction attacks~\citep{rigaki2023survey}, which recover sensitive data from learned parameters,
or membership inference attacks~\citep{shokri2017membership}, 
which determine whether a specific individual's record was included in a training data set.
Other types of attacks are described in~\citep{fredrikson2015model, ganju2018property, melis2019exploiting}. 
Differential privacy~\citep{dwork2014algorithmic} 
is a general-purpose privacy framework that protects against many such attacks, and 
differentially private ML has been used to train models in such a way that an adversary cannot 
reliably use a model to draw accurate conclusions about the data used in training it. 
In this paper, we focus on using differential privacy in training deep neural networks
for classification tasks. 

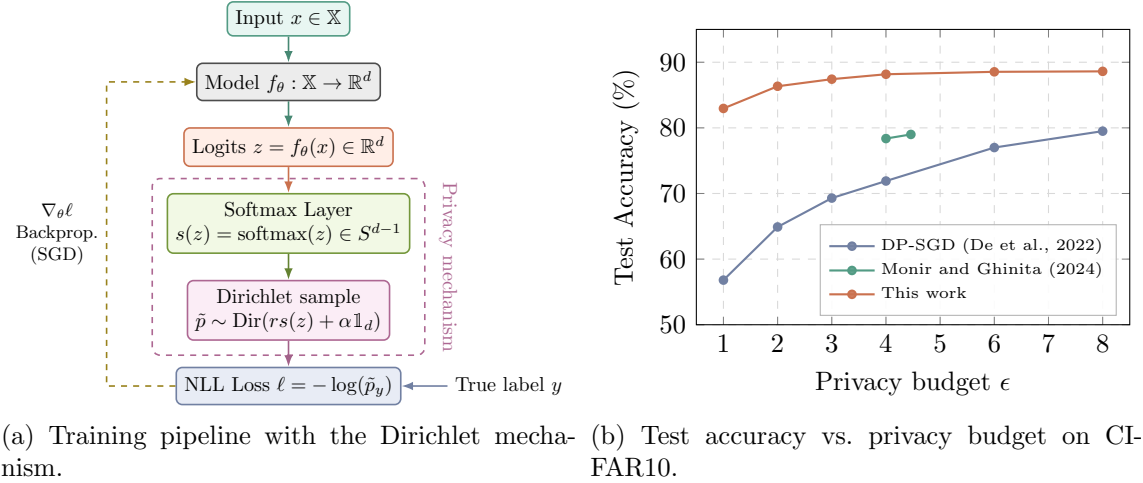
\begin{figure}[!t]
    \centering
    \hspace*{-10mm}% 
    \setlength{\figwidth}{0.5\textwidth}
    \setlength{\figheight}{5cm}
    % ============================================
% Figures/Intro/combine.tex
% ============================================

% ---- Colors ----
\providecolor{cbTeal}  {HTML}{66C2A5}
\providecolor{cbOrange}{HTML}{FC8D62}
\providecolor{cbPurple}{HTML}{8DA0CB}
\providecolor{cbPink}  {HTML}{E78AC3}
\providecolor{cbGreen} {HTML}{A6D854}
\providecolor{cbYellow}{HTML}{FFD92F}
\providecolor{cbGrey}  {HTML}{B3B3B3}

% ---- TikZ styles ----
\tikzset{
    box/.style={
        rectangle, rounded corners=3pt, draw, thick,
        align=center, inner sep=4pt, font=\small
    },
    input/.style    ={box, draw=cbTeal!80!black,   fill=cbTeal!15},
    model/.style    ={box, draw=black!70,          fill=cbGrey!25},
    logits/.style   ={box, draw=cbOrange!80!black, fill=cbOrange!15},
    expclip/.style  ={box, draw=cbGreen!70!black,  fill=cbGreen!15},
    dirich/.style   ={box, draw=cbPink!75!black,   fill=cbPink!15},
    loss/.style     ={box, draw=cbPurple!80!black, fill=cbPurple!20},
    fwd/.style      ={-{Latex[length=2mm]}, thick, cbTeal!70!black},
    fwd2/.style     ={-{Latex[length=2mm]}, thick, cbOrange!80!black},
    fwd3/.style     ={-{Latex[length=2mm]}, thick, cbGreen!60!black},
    fwd4/.style     ={-{Latex[length=2mm]}, thick, cbPink!70!black},
    fwd5/.style     ={-{Latex[length=2mm]}, thick, cbPurple!80!black},
    bwd/.style      ={-{Latex[length=2mm]}, thick, dashed, cbYellow!60!black}
}

\centering
% ============================================
% LEFT: Vertical pipeline diagram
% ============================================
% ============================================
% LEFT: Vertical pipeline diagram
% ============================================
\begin{subfigure}[t]{0.50\textwidth}
  \centering
  \resizebox{\linewidth}{!}{
  \begin{tikzpicture}[
      font=\small,
      node distance=4.5mm,
  ]
    \node[input]                    (input)  {Input $x\in \X$};
    \node[model,   below=of input]  (model)  {Model $f_\theta:\X \to \mathbb{R}^d$};
    \node[logits,  below=of model]  (logits) {Logits $z=f_\theta(x) \in \mathbb{R}^d$};
    \node[expclip, below=of logits] (exp)    {Softmax Layer\\$s(z)=\softmax({z}) \in S^{d-1}$};
    \node[dirich,  below=of exp]    (dir)    {Dirichlet sample\\$\tilde{p}\sim\mathrm{Dir}(rs(z)+\alpha\onevec_d)$};
    \node[loss,    below=of dir]    (loss)   {NLL Loss $\ell=-\log(\tilde{p}_{y})$};

    \node[font=\small, right=8mm of loss] (ystar) {True label $y$};

    \node[draw=cbPink!75!black, dashed, thick, rounded corners=4pt,
          inner sep=2.5mm, fit=(exp)(dir)] (priv) {};
    \node[cbPink!75!black, font=\small, right=1mm of priv.east,
          rotate=-90, anchor=south] {Privacy mechanism};

    \draw[fwd]  (input)  -- (model);
    \draw[fwd]  (model)  -- (logits);
    \draw[fwd2] (logits) -- (exp);
    \draw[fwd3] (exp)    -- (dir);
    \draw[fwd4] (dir)    -- (loss);
    \draw[fwd5] (ystar)  -- (loss);

    % UPDATED: Added (SGD) to the backprop node
    \draw[bwd] (loss.west)
        -- ++(-12mm,0) coordinate (b_bot)
        -- (b_bot |- model.west)
        node[midway, left, font=\footnotesize, black, align=center] {$\nabla_\theta\ell$\\Backprop.\\(SGD)}
        -- (model.west);
  \end{tikzpicture}
  }
  \caption{Training pipeline with the Dirichlet mechanism.}
  \label{fig:pipeline}
\end{subfigure}%
\hfill
% ============================================
% RIGHT: Accuracy vs epsilon plot
% ============================================
\begin{subfigure}[t]{0.48\textwidth}
  \centering
  \begin{tikzpicture}
  \begin{axis}[
      width=\linewidth,
      height=0.75\linewidth,
      xlabel={Privacy budget $\epsilon$},
      ylabel={Test Accuracy (\%)},
      xmin=0.5, xmax=8.5,
      ymin=50, ymax=95,
      xtick={1,2,3,4,5,6,7,8},
      ytick={50,60,70,80,90},
      grid=major,
      grid style={dashed, gray!30},
      legend pos=south east,
      legend cell align={left},
      legend style={font=\tiny, draw=gray!60,
                    fill=white, fill opacity=0.9, text opacity=1,
                    row sep=-1pt, inner sep=3pt},
      tick label style={font=\small},
      label style={font=\small},
      major tick length=3pt,
  ]
  \addplot[color=cbPurple!80!black, mark=*, mark size=1.4pt, thick]
      coordinates {(1,56.8)(2,64.9)(3,69.3)(4,71.9)(6,77.0)(8,79.5)};
  \addlegendentry{DP-SGD~\citep{de2022unlocking}}

  \addplot[color=cbTeal!80!black, mark=*, mark size=1.4pt, thick]
      coordinates {(4,78.37)(4.46,79.0)};
  \addlegendentry{\citet{monir2024differentially}}

  \addplot[color=cbOrange!80!black, mark=*, mark size=1.4pt, thick]
      coordinates {(1,82.96)(2,86.34)(3,87.42)(4,88.17)(6,88.54)(8,88.61)};
  \addlegendentry{This work}
  \end{axis}
  \end{tikzpicture}
  \caption{Test accuracy vs.\ privacy budget on CIFAR10.  
  }
  \label{fig:accuracy}
\end{subfigure}
    \caption{Overview of the proposed private training framework. (a) The model $f_\theta$ first outputs logits that pass through a softmax layer, then a private element of the unit simplex is sampled from a Dirichlet distribution, and it 
    is used to compute the Negative Log-likelihood (NLL) loss. Then the loss is used to compute gradients and backpropagate to update the model parameters~$\theta$. (b) On CIFAR10, our method consistently outperforms DP-SGD~\citep{de2022unlocking} and \citet{monir2024differentially} across all privacy budgets $\epsilon$. For example, our method achieves over $82$\% accuracy when $\epsilon=1$, 
    while the prior work~\citep{monir2024differentially} achieves $79\%$ accuracy when $\epsilon=4.46$.        
    }
\end{figure}

In some real-world applications, training inputs and the labels for those inputs do not have the same privacy risks. 
For example, in medical imaging, chest X-rays used for pneumonia detection may contain highly sensitive patient biometric information~\citep{packhauser2022deep}, whereas the associated diagnostic labels may already be accessible via public health registries~\citep{wang2017chestx}.
In recommender systems, users publicly disclose their preferences through item ratings, while their personal social and demographic information (e.g., gender, political affiliation, and ethnicity) may be considered sensitive~\citep{bhagat2014recommending}.
In financial applications, individual transaction histories are sensitive, while the categories such as risk tiers are 
often publicly defined. 
For example, in credit underwriting, borrowers are often publicly grouped into established tiers like ``Prime'', ``Near-Prime'', or ``Subprime''~\citep{bazarbash2019fintech}. 
In these and other applications, 
the training inputs often contain sensitive information about an individual, while the labels of that data are not sensitive
and/or may already be public. Motivated by such applications, 
the question we address is the following: 
\begin{quote}
\textit{Can we train ML models that provide
differential privacy to sensitive training inputs while maintaining high accuracy of the models?}
\end{quote}

In this paper, we introduce a novel differentially private learning framework that injects calibrated randomness 
into the outputs of \softmax layers, which enables model training with private neural network inputs and public labels. 
Our approach decomposes the training process into two stages: (i) computing a distribution that enforces a desired
level of differential privacy, and (ii) randomizing \softmax outputs at every training step by using 
the Dirichlet mechanism~\citep{ponnoprat2021dirichlet},
which perturbs softmax outputs when computing the training loss in order to enforce differential privacy.
An illustration of the proposed framework is shown in Figure~\ref{fig:pipeline}.
As shown in Figure~\ref{fig:accuracy}, our method significantly outperforms DP-SGD~\citep{de2022unlocking}
and other prior work~\citep{monir2024differentially} for all of the privacy budgets that were evaluated.

% \begin{figure}[!t]
%     \centering
%     \hspace*{-10mm}% 
%     \setlength{\figwidth}{0.5\textwidth}
%     \setlength{\figheight}{5cm}
%     \input{Figures/Intro/combine}
%     \caption{Overview of the proposed private training framework. (a) The model $f_\theta$ first outputs logits that pass through a softmax layer, then a private element of the unit simplex is sampled from a Dirichlet distribution, and it 
%     is used to compute the Negative Log-likelihood (NLL) loss. Then the loss is used to compute gradients and backpropagate to update the model parameters~$\theta$. (b) On CIFAR10, our method consistently outperforms DP-SGD~\citep{de2022unlocking} and \citet{monir2024differentially} across all privacy budgets $\epsilon$. For example, our method achieves over $82$\% accuracy when $\epsilon=1$, 
%     while the prior work~\citep{monir2024differentially} achieves $79\%$ accuracy when $\epsilon=4.46$.        
%     }
% \end{figure}

\subsection{Related Work}
Differentially private stochastic gradient descent (DP-SGD)~\citep{abadi2016deep} is a widely used differentially private ML algorithm, which clips per-sample gradients and adds calibrated Gaussian noise to the gradient at each training step. 
While DP-SGD provides strong privacy guarantees to both training inputs and their labels, 
it introduces a substantial privacy-utility tradeoff: even at moderate privacy budgets, there is a large accuracy gap between models trained with DP-SGD and their non-private baselines. 
Many follow-up works have sought to reduce the privacy-utility gap of DP-SGD~\citep{de2022unlocking,tramer2020differentially,bu2023automatic,sander2023tan, papernot2021tempered, cheng2022dpnas}.
One recent body of work has privatized only labels while leaving training inputs public,
which has been termed ``Label-DP''~\citep{ghazi2021deep, malek2021antipodes, busa2023label, ghazi2022regression, esfandiari2022label}. 
By assuming that training inputs are non-sensitive, Label-DP reduces the amount of noise injected during training and achieves substantially better accuracy than DP-SGD.

The complementary case, in which inputs are sensitive but labels are public, has received less attention.
Several works~\citep{fan2018image, croft2021obfuscation, luo2025dpo, xue2021dp} privatize input data before applying a training procedure, rather than developing a new training procedure that is designed to enforce privacy.
For example,~\citet{fan2018image} introduces a differentially private pixelization method that partitions an image into grid cells and perturbs each cell's average pixel value with Laplace noise.
However, the methods in~\citep{fan2018image, croft2021obfuscation, luo2025dpo, xue2021dp}
release a private image rather than a trained ML model, 
and their evaluations focus on re-identification resistance rather than classification accuracy. 
The goals of that work are different from the
current paper, and a direct comparison would be ill-defined.
The closest prior work to ours is~\cite{monir2024differentially}, which addresses the setting of sensitive inputs and public labels by injecting calibrated Gaussian noise into the penultimate layer of a deep neural network in a way similar to DP-SGD. 
However, the variance of privacy noise scales with layer width,
which can harm accuracy, 
and per-sample gradient clipping introduces additional computational cost.
We show in Section~\ref{sec:evaluation} that our approach provides a $45.31\%$ 
reduction in the classification error on CIFAR10
relative to~\cite{monir2024differentially} when $\epsilon=4.0$. 

Privacy amplification by subsampling~\citep{wang2019subsampled, mironov2019r} allows DP-SGD-style methods to use the 
moments accountant to obtain privacy guarantees over many training steps that are stronger than standard
composition-type results for differential privacy. 
The core idea is to apply a privacy mechanism to randomly sampled minibatches from data sets and track the log moment generating function of the privacy loss. Composing these moments ensures a tight bound on the accumulated privacy loss across all training steps.
Most existing work based on subsampling
uses the Gaussian mechanism, which is the standard mechanism used in DP-SGD.
To the best of our knowledge, there is no prior work that has used 
the Dirichlet mechanism~\citep{ponnoprat2021dirichlet} under subsampling.
Our work does so, and thus we address not only the setting
of learning from sensitive inputs with public labels, but we also provide the first subsampled privacy analysis of the Dirichlet mechanism.

\subsection{Summary of Contributions}
In this work, we make the following contributions:
\begin{itemize}
    \item We propose a novel differentially private training framework that generates private softmax outputs using the Dirichlet mechanism  (Algorithm~\ref{alg:FDP-SGD}).
    \item We provide a formal privacy analysis using subsampling amplification that gives a tight bound on the privacy guarantees of Algorithm~\ref{alg:FDP-SGD} (Theorem~\ref{theorem:MA}).
    \item We show that, on average, the implementation
    of privacy preserves the direction of gradients while scaling
    their magnitude, and
    we use this result to demonstrate how to select learning
    rates to negate this scaling (Theorem~\ref{thm:main}).   
    \item We empirically evaluate our method on several benchmark data sets     
    and show that it outperforms prior work across all evaluated privacy budgets (Section~\ref{sec:eval}). 
    A depiction of these results is shown in Figure~\ref{fig:accuracy}.    
    Our approach provides a $45.31\%$    
    reduction in the classification error on CIFAR10 relative to
    the method in~\citep{monir2024differentially} when $\epsilon=4.0$.
\end{itemize}

\subsection{Notation}
We use $\R$ to denote the real numbers,~$\R_{\geq0}$ to denote the non-negative real numbers, and $\mathbb{N}$ to denote the positive integers. For $n\in \mathbb{N}$, we use $[n]$ to denote the set $\{1, \dots, n\}$.
We use~$|S|$ to denote the cardinality of a finite set~$S$. 
For $z\in \mathbb{R}^d$, we use $s(z)$ to denote the softmax function with input $z$, where the $i^{th}$ component of the output of the softmax function is defined as
$s_i(z)=\frac{\exp(z_i)}{\sum_{j \in [d]}\exp(z_j)}$.
For $x>0$, the gamma function is defined as $\Gamma(x)=\int_0^\infty t^{x-1} e^{-t}dt$, 
and we write
$\psi(x) = \tfrac{d}{dx}\log(\Gamma(x))$ and
$\psi'(x) = \tfrac{d^2}{dx^2}\log(\Gamma(x))$ for the digamma and trigamma
functions, respectively.
We use $D_{KL}$ to denote the Kullback–Leibler (KL) divergence.
For ${v} \in \R^d_{\geq 0}$ the multivariate beta function is defined as $B(v)= {\prod\limits_{i=1}^{d} \Gamma(v_i)} \bigg/{\Gamma\left(\sum\limits_{i=1}^{d} v_i\right)}$.
For $q\in(0,1)$, we use $\text{Bernoulli}(q)$ to denote the Bernoulli distribution with parameter~$q$, 
where a random variable drawn from this distribution takes the value $1$ with probability $q$
and takes the value~$0$ with probability~$1-q$.
For a non-empty, finite set $S$, we use $\mathrm{Uniform}(S)$ to denote 
the uniform distribution over $S$, where each element $s\in S$ has probability $1/|S|$. 
We define $\onevec_d=[1,\dots, 1]^\top \in \mathbb{R}^d$.
By convention, we have~$\binom{n}{m} = 0$ if~$n < m$. 
We use the Kronecker delta function~$\delta_{ij}$,
where~$\delta_{ij} = 1$ if~$i = j$ and~$\delta_{ij} = 0$ if~$i \neq j$. 
We use~$\textnormal{Cov}(\cdot, \cdot)$ to denote the covariance of two random variables.
Lastly, $\|\cdot\|_2$ is the $\ell^2$ norm, and $\|\cdot\|_\infty$ is the $\ell^\infty$ norm, both are norms on Euclidean spaces.

\section{Preliminaries}
In this section, we review the background needed for the rest of the paper. 
The final layer 
of a classification network typically outputs a vector of real-valued scores called logits.
These logits can be any real number, which makes them difficult to interpret as probabilities
that an input belongs to each class. 
The softmax function maps logits to a probability distribution over classes, i.e., a vector with non-negative entries that sum to one.
Such a vector is an element of the unit simplex. 

\begin{definition}[Unit Simplex]
    Let $d\in \mathbb{N}.$ The unit simplex in $\R^d$ is denoted by $S^{d-1}$ and
    is defined as the set of all element-wise non-negative vectors of length $d$ whose entries sum to $1$, i.e., 
    \begin{equation}
        S^{d-1} = \biggl\{ p \in \R^d : \sum_{i=1}^d p_i = 1, p_j \geq 0 \ \text{for all} \ j \in [d] \ \biggr\}.
    \end{equation}
\end{definition}
We will sometimes refer to elements of the simplex as ``probability vectors''. 

We consider data sets of the form $\D=\{(x_i,y_i) \}_{i=1}^n$ for some~$n \in \mathbb{N}$, where $x_i$ is a training input and $y_i$ is its label for each~$i \in [n]$. 
Let~$\DD=\X \times \Y$ denote the collection of all data sets of interest, where $\X$ is the collection of all training inputs and $\Y$ is the collection of all labels. 
We do not make any assumptions that~$\DD$ is a particular type of space (e.g., Euclidean).
The goal of differential privacy is to 
randomize data in a way that
makes two ``similar'' pieces of sensitive data produce outputs that are ``approximately indistinguishable''. 
Given two data sets $\D$ and $\D'$, the notion of ``similarity" is formalized by an adjacency relation~\citep{dwork2014algorithmic}.

\begin{definition}[Adjacent Data Sets]
\label{def:adj}
    Let $\D, \D' \in \DD$.
    Then $\D$ and $\D'$ are \emph{adjacent} if one can be obtained from the other by adding or removing a single data point; that is, either $\D' = \D \cup \{(x, y)\}$ or $\D = \D' \cup \{(x, y)\}$ for some data point $(x, y)$. We write $\Adj(\D, \D')=1$ if $\D$ and $\D'$ are adjacent and $\Adj(\D, \D')=0$ otherwise.
\end{definition}

Definition~\ref{def:adj} states that two data sets are adjacent if they differ in a single data point.
The notion of ``approximately indistinguishable'' is made precise by the definition of differential privacy itself.

\begin{definition}[Differential Privacy, \citealt{dwork2014algorithmic}]
\label{def:dp}
Let $\delta \in [0, 1)$ and $\epsilon > 0$ be given. A randomized mechanism $\M: \DD  \to \Simplex$ is $(\epsilon, \delta)$-differentially private if for every pair of data sets $\D,\D'\in \DD$ satisfying $\Adj(\D, \D')=1$ and every measurable output set $\event\subseteq\Simplex,$
we have
\begin{equation}
    \prob[\M(\D) \in \event] \leq e^{\epsilon} \cdot \prob[\M(\D')\in \event] + \delta.
\end{equation}
If $\delta=0$, then $\M$ is said to be $\epsilon$-differentially private.
\end{definition}
In Definition~\ref{def:dp}, smaller values of $\epsilon$ and $\delta$ correspond to stronger privacy guarantees.
Typical values of~$\epsilon$ are $0.1$ to $10$~\citep{hsu2014differential}, and typical values of $\delta$ are chosen to be 
much less than $\frac{1}{n}$. 
An appealing property of differential privacy is that it is immune to post-processing 
in the sense that 
arbitrary computations  on differentially private data do not weaken its privacy guarantees.
Given a data set $\D=\{(x_i,y_i) \}_{i=1}^n \in \DD$, suppose that each entry belongs to an individual and the data set is used to train 
a model while keeping~$\D$ differentially private. 
Then, Definitions~\ref{def:adj} and~\ref{def:dp} guarantee that given the trained model or an output it generates, an adversary cannot reliably determine if a specific individual's data is in~$\D$. 

We next introduce \Renyi differential privacy (RDP), which is a generalization of~$(\epsilon, \delta)$-differential privacy.

\begin{definition}[\Renyi Differential Privacy (RDP), \citealt{mironov2017renyi}]\label{def:RDP}
   Fix $\lambda>1$. A randomized mechanism $\M: \DD \to \Simplex$ is $(\lambda, \hat{\epsilon})$-\Renyi differentially private if for all pairs of data sets $\D, \D' \in \DD$ satisfying $\Adj(\D,\D')=1$, we have 
    \begin{equation}
        D_\lambda \bigl(\M(\D) \,\|\, \M(\D') \bigr):= \frac{1}{\lambda-1} \log \E_{y\sim \mathcal{M}(\D)} \biggl[\left(\frac{p_{\mathcal{M}(\D)}(y)}{p_{\mathcal{M}(\D')}(y)}\right)^{\lambda-1} \biggr]\leq \hat{\epsilon},
    \end{equation}
    where $p_{\mathcal{M}(\cdot)}$ denotes the probability distribution of the output of $\M(\cdot).$
    For $\lambda=1$, we define $D_1 \bigl(\M(\D) \,\|\, \M(\D') \bigr) = D_{KL} \bigl(\M(\D) \,\|\, \M(\D') \bigr).$    
\end{definition}

In Definition~\ref{def:RDP}, smaller values of 
$\hat{\epsilon}$ and larger values of $\lambda$ both correspond to stronger privacy guarantees.

At each training step, our method applies a privacy mechanism to 
provide differential privacy to 
a random minibatch rather than the full data set.
Intuitively, for two adjacent data sets $\D$ and $\D',$ the differing record between the data sets is unlikely to be
included in any given minibatch, 
which amplifies the protections of differential privacy. 
This idea is formalized by $\ps$, which is often used in private learning settings.

\begin{definition}[PoissonSample,~\citealt{zhu2019poission}]\label{def:poisson}
    Consider a data set $\D=\{(x_i, \\ y_i)\}_{i=1}^{n} \in \DD$ and select~$q \in (0, 1)$.
    Then, the procedure $\ps$ outputs a random subset~$\D_B \subseteq \D$ with expected cardinality~$q|\D|$ by:
    (i) sampling $\sigma_i \sim \text{Bernoulli}(q)$ independently for each $i\in[n]$ and 
    (ii) outputting the set $\D_B=\{(x_i, y_i) \in \D : \sigma_i=1\}$. 
\end{definition}

Since~$\mathbb{E}\big[|\D_B|\big] = q|\D|$, we refer to~$\D_B$ as a ``$q$-proportion'' of~$\D$. 
When we apply an $(\epsilon,\delta)$-differentially private mechanism to a random $q$-proportion of the data set, 
privacy amplification ensures the mechanism provides 
$\bigl(\log(1+q(e^\epsilon-1))$, $q\delta \bigr)$-differential privacy
to the full data set.
It is straightforward to show that this latter privacy protection is
$(O(q\epsilon), q\delta)$-differential privacy, which is stronger
than~$(\epsilon, \delta)$-differential privacy because~$q \in (0, 1)$. 
This improved privacy guarantee 
can be combined with the composition of RDP to tightly characterize 
the privacy loss over many applications
of a differential privacy mechanism. 
As we show in the next section, 
by composing the privacy loss in the sense of RDP before converting back to standard $(\epsilon, \delta)$-differential privacy, we can run our mechanism many times
(and hence train a neural network for many epochs)
without blowing up the privacy budget $\epsilon$.

\section{Proposed Algorithm}
In this section, we develop a private training framework that provides privacy to training inputs 
while leaving labels public.
During training, our method perturbs the outputs
of a softmax layer by drawing samples from a Dirichlet distribution
to enforce differential privacy. 
Each perturbed output 
can be interpreted as 
a vector of 
perturbed probabilities that
a training data point belongs to each class of data. 
Then these perturbed probabilities and the 
associated 
true label are used to compute the loss to be optimized.

We begin by introducing the Dirichlet mechanism. Then we provide the details of our framework and analyze its privacy guarantees.
Lastly, we analyze the first and second moments of the perturbed loss to interpret how the training process changes with privacy.

\subsection{Dirichlet Mechanism}
\label{sec:the_dir_mech}
A mechanism 
is a map that 
enforces differential privacy by randomizing functions of sensitive data to guarantee that Definition~\ref{def:dp} is satisfied. A mechanism is calibrated using the ``sensitivity'' of the function being privatized, which 
is equal to the maximum amount that function's outputs can differ when it is evaluated 
on adjacent inputs. 

\begin{definition}[Sensitivity]
\label{def:sensitivity}
    Fix a function $h : \X \to \Simplex$. The $\ell^2$- and $\ell^\infty$-sensitivities of $h$, respectively, are
    \begin{equation}
        \Delta_2 =  \sup_{\substack{x, x'\in \X}} \|h(x) - h(x')\|_2
    \end{equation}
    and
    \begin{equation}
        \Delta_\infty =  \sup_{\substack{x, x'\in \X}} \|h(x) - h(x')\|_\infty.
    \end{equation}    
\end{definition}

In our setting, the sensitive data is the training inputs in the data set $\D$, and the function we privatize is the mapping from neural network inputs to softmax outputs.
We privatize at the softmax layer because, regardless of the architecture or current parameters of the neural network, 
its output lies in the unit simplex, and the compactness of the simplex immediately bounds the 
neural network's sensitivity. We next make this observation precise. 

% \calvin{Only other $\theta(\D)$ is here, I think we can delete.}
% Mathematically, a
% model is a mapping $f_\theta: \X\to\mathbb{R}^d$
% parametrized by $\theta$. 
% We sometimes make the dependence of the parameters on the training data explicit by writing 
% $\theta(\D)$ for the parameters produced when training on a data set $\D$.
% With an abuse of notation, we also use~$\theta(\D)$ generically
% to mean the values of training parameters that one may have at an
% arbitrary iteration of training a neural network, and the meaning will always be clear from context. 
% \mh{Later, when we talk about Algorithm~1 I think we should use~$\theta_t$ instead of~$\theta(\D)$.
% Any objections to cutting the notation ``$\theta_t$''?
% }

% Throughout training, we will apply the Dirichlet mechanism many times.
% \hr{\sout{The only facet of each application 
% that depends on the sensitive data $\D$ is the function $g_{\theta(\D)}(x) = \softmax \left(f_{\theta(\D)}(x)\right) \in \Simplex$, 
% which produces a probability vector in the simplex for any input $x\in\X$.
% The two sensitivities we require are formally defined as follows.}}

Mathematically, a
neural network defines a mapping $f_\theta: \X\to\mathbb{R}^d$
parametrized by $\theta,$ where $f_\theta(x)$ denotes the logits associated with an input $x.$
Let $s:\mathbb{R}^d\to\Simplex$ denote the softmax map.
At a fixed training step, the current model parameters $\theta$ induce a map 
\begin{equation}
    g_\theta=s\circ f_\theta:\X\to\Simplex
\end{equation}
so that $g_\theta(x)=s(f_\theta(x))$
maps the training input $x$ to output probabilities.
We randomize the outputs of~$g_{\theta}$ to privatize
its inputs, and therefore we require the sensitivities of~$g_{\theta}$ to implement privacy. 
% For a training data set $\D=\{(x_i, y_i)\}_{i=1}^{n}$, define the collection of $\softmax$ outputs
% \begin{equation}
%     C_\theta(D) = \{g_\theta(x_1), \dots, g_\theta(x_n) \} \in S^{(d-1)\times n}.
% \end{equation}
% Since the labels are public, the private object at each training step is \hr{a subset of?} $C_\theta(D)$, while the labels are used as side information to compute the loss function.

\begin{lemma} \label{lem:sensitivities}
    Let~$g_{\theta} : \X \to \Simplex$ be the model associated with any neural network
    with a softmax output layer. Then, regardless of the architecture and model parameters, its 
    sensitivities in the sense of Definition~\ref{def:sensitivity} are bounded via 
    $\Delta_2^2 \leq 2$ and $\Delta_\infty \leq 1$.
\end{lemma}

\begin{proof}
    See Appendix~\ref{appendix:proof_sen}.
\end{proof}

We now introduce the Dirichlet mechanism.
\begin{mechanism}[{Dirichlet Mechanism,~\citealt[Section 3.1]{ponnoprat2021dirichlet}}]
    \label{mech:Dirichlet}
     Fix a function $g_{\theta}:\X \to S^{d-1}$ parameterized by $\theta$,
     and fix constants~$r > 0$ and~$\alpha > 0$. 
     Then, for any input $x \in \X$, the Dirichlet Mechanism $\mathcal{M}_{\text{Dir}}^{(r,\alpha)}: S^{d-1} \to S^{d-1}$ takes  $g_\theta(x)$ as input 
     and outputs $\rho \in S^{d-1}$ sampled from the Dirichlet distribution via 
    \begin{equation}
        \rho \sim \emph{Dir}(rg_\theta(x) + \alpha \onevec_d),
    \end{equation}
    where the pdf of~$\emph{Dir}(v)$ is 
    $f(p;v)=\frac{1}{B(v)}\prod_{i=1}^{d}p_i^{v_i-1}$.  
\end{mechanism}

The Dirichlet mechanism enforces differential privacy
with strength determined by~$r$ and~$\alpha$. 
We will sometimes refer to~$\alpha$
as an ``offset parameter''
and~$r$ as a ``scaling parameter''. 

\begin{lemma}[\citealt{ponnoprat2021dirichlet}] \label{lem:mech1rdp}
    Fix constants $\alpha>0$ and $r\in(0,\alpha)$, and
    fix an RDP order $\lambda \in[1, 1+\frac{\alpha}{r})$.
    Mechanism~\ref{mech:Dirichlet} enforces $(\lambda, \hat{\epsilon}(\lambda;r))$-RDP
    as defined in Definition~\ref{def:RDP}, 
    where
    \begin{equation}
        \hat{\epsilon}(\lambda;r) = \lambda r^2 \psi'(\alpha-(\lambda-1)r). 
    \end{equation}
\end{lemma}

In our analysis, we adapt the RDP bound for $\ps$ from~\citet{zhu2019poission}, in which the underlying mechanism is fixed while its RDP guarantee is evaluated across different values of $\lambda$.
Therefore, in our training procedure we fix the Dirichlet mechanism parameters $\alpha$ and $r$ independently of $\lambda$.

The privacy guarantee we ultimately seek is for the training inputs within the training set $\D$, i.e.,
for a point~$(x_i, y_i) \in \D$ we seek to privatize~$x_i$. 
The sensitivities~$\Delta_2^2 \leq 2$ and~$\Delta_{\infty} \leq 1$ 
from Lemma~\ref{lem:sensitivities} 
bound how much one forward pass 
can change between two adjacent training sets, 
which we will use 
in Section~\ref{subsec:analysis}
to bound the privacy loss for the entire set of input data.

\subsection{Private Training Algorithm}
In this subsection, we present a differentially private training framework that privatizes training inputs while leaving labels public.
Algorithm~\ref{alg:FDP-SGD} outlines the process for training a model with a \softmax output layer by minimizing an empirical loss function $\loss$. 
Our complete differentially private training pipeline is as follows. 
At each training step $t$, we sample a minibatch $\D_{B_t}$ with sampling rate $q$ by using $\ps$
from Definition~\ref{def:poisson}.
Then, for each~$i \in B_t$
we compute the output probability vector $g_{\theta_t}(x_i)$ and sample a private probability vector $\tilde{p}^i$ from $\text{Dir}(rg_{\theta_t}(x_i)+\alpha\onevec_d)$, where~$\alpha > 0$ and~$r \in (0, \alpha)$ are constants that we choose below. 
Next, we compute the averaged noisy gradient over the minibatch, which is $\tilde{\xi}=\dfrac{1}{|B_t|} \sum_{i \in B_t} \nabla_{\theta} \loss(\tilde{p}^i, y_i)$, where the vector $\tilde{p}^i$ is used in combination with the true label $y_i$ to compute the NLL\footnote{We use NLL rather than cross-entropy (CE) loss because our mechanism
directly outputs probabilities and typical implementations of the CE loss operate on the logits directly.
Without privacy, using a \softmax layer with NLL loss is equivalent to using the CE loss on the logits.} loss $\mathcal{L}(\tilde{p}^i,y_i)=-\log\tilde{p}_y^i$.
Finally, the algorithm takes a step in the direction~$-\tilde{\xi}$.
After training for a fixed number of steps $T$, we output the trained model.

This form of training with the Dirichlet mechanism can potentially require many iterations, and the strength of privacy degrades with each iteration
because training data can be reused. 
If the Dirichlet mechanism is applied to enforce~$(\epsilon, \delta)$-differential privacy~$T$ times, then the standard Composition Theorem~\citep[Chapter 3.5]{dwork2014algorithmic} shows that the data is protected overall with~$(T\epsilon, T\delta)$-differential privacy. For large~$T$, these protections may become weak. To reduce the
rate of weakening of 
privacy protections, we use \emph{privacy amplification} via $\ps$~\citep{zhu2019poission}, which we describe in detail in Section~\ref{subsec:analysis}.

Our approach to privacy is optimizer-agnostic in the sense that any optimizer can be used without harming privacy.
However, Algorithm~\ref{alg:FDP-SGD} is written with stochastic gradient descent (SGD) 
because SGD is widely used, and that is the algorithm we analyze in
Section~\ref{sec:analysis}.

\subsection{Privacy Analysis} \label{subsec:analysis}
To analyze privacy we observe that in Algorithm~\ref{alg:FDP-SGD} backpropagation only post-processes private data. 
Given a minibatch $\D_{B_t}$ at iteration~$t$, lines~$3$-$6$ of Algorithm~\ref{alg:FDP-SGD} compute the set
\begin{equation}
    \mathcal{E}_{B_t} = \{\tilde{p}^i : 
    \tilde{p}^i \sim\text{Dir}(rg_{\theta_t}(x_i)+\alpha\onevec_d), 
    (x_i, y_i) \in \D_{B_t} \textnormal{ for some } y_i\}, %\{g_{\theta}(x) : (x, y) \in \D \textnormal{ for some y}\},
\end{equation}
and the use of the Dirichlet mechanism ensures that the 
computation of $\mathcal{E}_{B_t}$ keeps~$x_i$ differentially private for each~$(x_i, y_i) \in \D_{B_t}$. 
%Since postprocessing preserves the differential privacy guarantee of its input, and backpropagation itself is a postprocessing function applied to $\mathcal{E}$, it follows that backpropagation also preserves privacy.
In Algorithm~\ref{alg:FDP-SGD}, 
backpropagation uses only~$\mathcal{E}_{B_t}$ and not~$\mathcal{D}_{B_t}$, which means that backpropagation
post-processes differentially private data. Therefore, the results of backpropagation
preserve differential privacy as well. 
%We now analyze the privacy guarantees of Algorithm~\ref{alg:FDP-SGD}.
The strength of privacy provided by each iteration of Algorithm~\ref{alg:FDP-SGD}
is therefore equal to the strength of privacy provided to the training
data when computing~$\mathcal{E}_{B_t}$ in lines~$3$-$6$ of Algorithm~\ref{alg:FDP-SGD}.
We use this fact in the next theorem, which is our main result on the privacy
guarantees of Algorithm~\ref{alg:FDP-SGD}.

\begin{theorem}
\label{theorem:MA}
   Fix a sampling rate $q\in(0,1)$ for $\ps$. Fix $\alpha>0$
   and~$r \in (0, \alpha)$. 
   Then for every integer~$\lambda$ satisfying $2\leq\lambda < 1+\frac{\alpha}{r}$,     
   Algorithm~\ref{alg:FDP-SGD} is $(\epsilon_o(\lambda,\delta;r),\delta)-$differentially private with
    \begin{equation}
    \begin{aligned}
        \epsilon_o(\lambda,\delta;r) &\leq \frac{T}{\lambda-1} \log \biggl \{ (1-q)^{\lambda-1}(q\lambda-q+1)   + \binom{\lambda}{2} q^2 (1-q)^{\lambda-2} e^{\hat{\epsilon}_\M(2;r)} \\
        & + 3\sum^\lambda_{j=3}\binom{\lambda}{j}(1-q)^{\lambda-j}q^je^{(j-1)\hat{\epsilon}_\M(j;r)} \biggr \} + \log(\lambda-1) - \frac{\log(\delta)+\lambda \log(\lambda)}{\lambda-1},        
    \end{aligned}
    \end{equation}
    where $\hat{\epsilon}_\M(j;r) = j r^2\psi'(\alpha-(j-1)r)$ and $T$ is the number of training steps.
\end{theorem}

\begin{proof}
    See Appendix~\ref{appendix:proof_1}.
\end{proof}

Once the parameters $r$, $\lambda$, and $\delta$ are fixed, the privacy parameter 
$\epsilon_o(\lambda, \delta; r)$ grows as~$q$ grows,
which is intuitive because a larger sampling rate~$q$
means more points are sampled and that
there is less benefit from privacy amplification.
Therefore, a small value of $q \in (0, 1)$ should be chosen so 
that the argument of the first~$\log$ term in Theorem~\ref{theorem:MA} is dominated by its first two terms.
Doing so slows the rate of decay of privacy protections as~$T$ grows.

The values of $r, \alpha$, $\delta$, $q$, and $T$ are always fixed at the beginning 
of Algorithm~1 
so that $\epsilon_o(\lambda, \delta;r)$ is only a function of the RDP order $\lambda$. 
In practice $\lambda$ is always selected to achieve the minimum value of $\epsilon_o(\lambda, \delta;r)$.

\begin{savenotes}
\begin{algorithm}[!t]
\caption{Training deep neural network classifiers
with end-to-end differential privacy}\label{alg:FDP-SGD}
\begin{algorithmic}[1]
\STATEx \textbf{Input:} Data set $\D = \{(x_i, y_i)\}_{i=1}^{n}$, NLL loss function $\loss$.
\STATEx \textbf{Parameters:} Learning rate $\gamma>0$, sampling rate $q\in(0,1)$, number of training steps $T\in\mathbb{N}$, Dirichlet mechanism parameters $\alpha>0$ and $r\in(0,\alpha)$.
\STATEx \textbf{Initialize} $\theta_0$ using the default PyTorch initialization for each layer\footnote{For network parameters $\theta_0$, weights and biases of linear and convolutional layers are drawn from 
the uniform distribution on the interval~$(-\frac{1}{\sqrt{k}}, \frac{1}{\sqrt{k}})$, 
where $k$ is the number of input features per neuron, normalization scaling factors are set to $1$, and normalization biases are set to $0$.}. 
\FOR{$t = 0, 1, \ldots, T-1$}
    \STATE Sample a mini-batch $\D_{B_t} \subseteq \D$ via $\ps$ 
    from Definition~\ref{def:poisson}
    with rate $q$.
    \FOR{each $(x_i, y_i) \in \D_{B_t}$}
        \STATE \textbf{Compute output probabilities:} Set $g_{\theta_t}(x_i) = \softmax(f_{\theta_t}(x_i))$.
        \STATE \textbf{Perturb output probabilities:} Sample $\tilde{p}^i \sim\text{Dir}(rg_{\theta_t}(x_i) + \alpha\onevec_d)$. 
    \ENDFOR
    \STATE \textbf{Compute gradient:} $\tilde{\xi}_t \gets \dfrac{1}{|B_t|} \sum_{i \in B_t} \nabla_{\theta_t} \loss(\tilde{p}^i, y_i)$.

    \STATE \textbf{Update parameters:} $\theta_{t+1} \gets \theta_t - \gamma \tilde{\xi}_t$.
\ENDFOR
\STATEx \textbf{Output:} $\theta_T$ 
\end{algorithmic}
\end{algorithm}
\end{savenotes}

\subsection{Selecting Dirichlet Mechanism Parameters}
By Theorem~\ref{theorem:MA}, Algorithm~\ref{alg:FDP-SGD} 
is~$(\epsilon,\delta)$-differentially private for some~$\epsilon$ and~$\delta$. 
Next, for given privacy parameters~$\epsilon > 0$ and~$\delta > 0$, 
we seek to select~$\alpha > 0$ and~$r \in (0, \alpha)$ 
so that Algorithm~\ref{alg:FDP-SGD} is $(\lambda, \hat{\epsilon}(\lambda;r))$-\Renyi differentially private
for values of~$\lambda$ and $\hat{\epsilon}(\lambda;r)$ that imply that~$(\epsilon, \delta)$-differential privacy
holds for the chosen values of~$\epsilon$ and~$\delta$.

Our approach is to first fix the offset parameter $\alpha > 0 $ and then solve for the largest scaling parameter $r > 0$ that ensures the enforcement
of~$(\epsilon, \delta)$-differential privacy. 
For any $r > 0$ the admissible RDP orders are
\begin{equation}
    \Lambda(r)=\left\{\lambda\in\mathbb{N}\ |\ 2\leq\lambda < 1+\frac{\alpha}{r} \right\},
\end{equation}
and we define $h(r)=\min_{\lambda\in\Lambda(r)}\epsilon_o(\lambda,\delta;r).$
Theorem~\ref{theorem:MA} implies that running Algorithm 1 for $T$ steps is $(h(r),\delta)-$differentially private.
To find the largest $r$ that the privacy budget allows we compute
\begin{equation}
\label{eq:optim}
    r^*=\sup\left\{ r\in(0,\alpha)\ |\ h(r)\leq \epsilon\right\}.
\end{equation}
Since $h(r)$ is increasing in $r$ and $\lim_{r \to \alpha} h(r)=\infty$, we see that for any $\epsilon > \lim_{r \to 0} h(r)$ there is an $r < \alpha$ such that $h(r)\leq \epsilon$.
Such an~$r$ can be found, for example, via root-finding on the equation $h(r)=\epsilon$. 
In our experiments 
in Section~\ref{sec:eval} 
we solve this root-finding problem with Brent's method~\citep{brent1971algorithm}.

\subsection{Analytical Characterization of Algorithm~\ref{alg:FDP-SGD}} \label{sec:analysis}
We now characterize how the choices of the 
mechanism parameters $(r,\alpha)$ affect the optimization steps in Algorithm~\ref{alg:FDP-SGD}.
Although Algorithm~\ref{alg:FDP-SGD} generates a sequence of parameters $\{\theta_t\}_{t=0}^T,$ in this subsection we analyze the effect of privacy for arbitrary network 
parameters $\theta.$
All  results in this subsection apply at any training step 
with index~$t \in [T] \cup \{0\}$ 
by setting $\theta=\theta_t.$
Since $r$ is automatically determined by~\eqref{eq:optim} in terms of a fixed $\alpha$ and 
fixed privacy budget $(\epsilon,\delta)$, 
the behavior of Algorithm~\ref{alg:FDP-SGD} is 
governed by $\alpha$ and the learning rate $\gamma$, and our
analysis is in terms of these parameters. 
The main result of this section is that the perturbations 
due to the Dirichlet mechanism
do not change the direction of the expected gradient
{of the per-sample loss of Algorithm~\ref{alg:FDP-SGD}}. 
Instead,
these perturbations 
result in an expected private gradient that is equal to a
rescaled version of 
the non-private cross-entropy gradient,
and this scaling is captured by a data-dependent scale 
factor.

Consider fixed network parameters $\theta$
and a training example~$(x, y) \in \D$. 
Let~$z = f_{\theta}(x)$ be the logits associated with~$x$.
Recall that $s(z)=\softmax(z)$ and $g_\theta(x)=s(f_\theta(x))=s(z).$
We use 
$s_j(z)$ 
to denote the probability assigned to class $j\in[d]$ and $s_y(z)$ to denote the probability assigned to the true class $y.$

The cross-entropy loss 
without privacy 
is $\loss_{\mathrm{CE}}(z) = -\log s_y(z)$, which has gradient 
$\tfrac{\partial \loss_{\mathrm{CE}}}{\partial z_j}(z) = s_j(z) - \delta_{yj}$ for $j\in[d].$
Algorithm~\ref{alg:FDP-SGD} perturbs $s(z)$ by drawing a Dirichlet sample and substituting it into
the negative log-likelihood loss to compute $\loss(\tilde p(z),y) = -\log\tilde p_y$, where $\tilde p(z) \sim \text{Dir}\bigl(  r s(z) + \alpha\onevec_d \bigr)$ and $r, \alpha > 0$ are the mechanism parameters.
% we write~$\tilde{p}$ as~$\tilde{p}(z)$ when
% it is useful to emphasize its dependence on
% the logits~$z$.

To characterize 
training using
Algorithm~\ref{alg:FDP-SGD}, we analyze the first two moments of the stochastic loss.
The expected loss 
\begin{equation} \label{eq:Fdef}
F(z) := \E_{\tilde p}[\loss(\tilde p(z), y)]
\end{equation}
and the expected gradient $\E_{\tilde p}\left[\frac{\partial}{\partial z_j} \loss(\tilde p(z), y)\right]$ quantify 
a training data point's 
expected influence on the gradient in Line 7 of Algorithm~\ref{alg:FDP-SGD}, 
while $\mathrm{Var}_{\tilde p}\bigl[\loss(\tilde p(z),y)\bigr]$ measures the variability of the loss around its mean.

\begin{theorem}
\label{thm:main}
Fix a data set $\D \in \DD$ and consider training a model using Algorithm~\ref{alg:FDP-SGD} with $r$, $\alpha > 0$.
Consider a single training example $(x,y)\in \D$ with logits $z=f_\theta(x)$ and $s(z) = \softmax\left(z\right)$.
Define the scalar attenuation factor
\begin{equation}
  \kappa(z; r, \alpha)
  := r\, s_y(z)\, \psi'\bigl(r s_y(z) + \alpha\bigr).
  \label{eq:kappa-def-main}
\end{equation}
%where $g_{\theta,y}(z)$ is the predicted probability of~$x$ belonging to class~$y$ under the network parameters~$\theta$. 
Then for every $z \in \mathbb{R}^d$ and every $j \in [d]$, for~$F$ in~\eqref{eq:Fdef} we have 
\begin{align}
  F(z)
    &= \psi(r + d\alpha) - \psi\bigl(r s_y(z) + \alpha\bigr)
    \label{eq:F-closed-main}\\
  \E_{\tilde p}\left[\frac{\partial}{\partial z_j} \loss(\tilde p(z), y)\right]
    &= \kappa\bigl(z; r, \alpha\bigr)\frac{\partial \loss_{\mathrm{CE}}}{\partial z_j}(z)
    \label{eq:gradF-closed}\\
  \mathrm{Var}_{\tilde p}\bigl[\loss(\tilde p(z), y)\bigr]
    &= \psi'\bigl(r s_y(z) + \alpha\bigr) - \psi'(r + d\alpha).
    \label{eq:var-closed}
\end{align}
\end{theorem}

\begin{proof}
    See Appendix~\ref{proof:thm2}.
\end{proof}
The last three equations
in Theorem~\ref{thm:main} have simple interpretations.
The first shows that the expected private loss depends on the logits only through the predicted true class probability $s_y(z).$
The second is the key identity, and it shows that, on a single data point~$(x, y)$,  
the expected gradient of the loss with respect to the logits is equal to the clean, non-private cross-entropy gradient multiplied by the scaling factor $\kappa(z;r,\alpha).$
Thus, on a per-sample
basis, our privacy implementation preserves the gradient's direction and only rescales its magnitude.
The third equation quantifies the randomness introduced by the Dirichlet mechanism through the variance of the stochastic loss.

\subsection{Selecting Learning Rates}

Theorem~\ref{thm:main} can be used to calibrate the learning rate of Algorithm~\ref{alg:FDP-SGD}.
In Line~8 of Algorithm~\ref{alg:FDP-SGD}, the parameters 
of the neural network
are updated with the average gradient over a batch. 
%and as discussed in Section~\ref{sec:analysis} each training example can have a different value of $\kappa(z;r,\alpha)$. 
Using Theorem~\ref{thm:main}, 
we compute 
the expected gradient 
evaluated on the output of the Dirichlet mechanism
in terms of the logits~$z$ produced by a single training
sample~$x$. This gradient
is scaled by $\kappa(z;r,\alpha)$ from~\eqref{eq:kappa-def-main}, 
and
the corresponding expected 
SGD update 
is
\begin{equation}
     \E_{\tilde{p}}\left[-\gamma\nabla\loss(\tilde p(z), y)\right]=-\gamma \kappa(z;r,\alpha)\nabla  \loss_{\mathrm{CE}}(z),
\end{equation}
% is
% \begin{equation}
%     -\gamma\nabla F(z)=-\gamma \kappa(z;r,\alpha)\nabla  \loss_{\mathrm{CE}}(z),
% \end{equation}
where~$\nabla  \loss_{\mathrm{CE}}(z)$ is the gradient of the non-private cross-entropy loss evaluated on the 
logits~$z$ of the 
training example~$x$.

Therefore, on average, a gradient step with privacy implemented
with learning rate $\gamma$ behaves the same as a gradient step without privacy
with learning rate 
\begin{equation} \label{eq:JEFFISHERE}
\gamma_{\mathrm{eff}}=\gamma\kappa(z;r,\alpha),
\end{equation}
which we call the \emph{effective learning rate} of such an implementation.
The attenuation factor~$\kappa(z; r,\alpha)$ is sample-dependent through its dependence on $z$, 
and we of course cannot have a sample-dependent learning rate.
However, the following empirical results show that if we simply use
the value $s_y(z)=0.9$ for all examples and compare training runs using different values of $\alpha$ but the same value of $\gamma_{\mathrm{eff}}=\gamma\kappa(z; r,\alpha)$,
then we achieve similar performance across those values of~$\alpha$.
These results imply that for varying values of the parameter $\alpha$, the learning rate $\gamma$ of Algorithm~\ref{alg:FDP-SGD} 
can be tuned to counter the effects of privacy and achieve a nearly constant level of accuracy.

\begin{figure}[t]
    \centering
    \setlength{\figwidth}{0.48\textwidth}
    \setlength{\figheight}{5cm}
    % Auto-generated by plot_bigger_run.py --tikz-out
% Requires: \usepackage{pgfplots}\usepgfplotslibrary{groupplots}
\definecolor{alphaCol0p1}{HTML}{440154}
\definecolor{alphaCol0p2}{HTML}{472D7B}
\definecolor{alphaCol0p3}{HTML}{3B528B}
\definecolor{alphaCol0p5}{HTML}{2C728E}
\definecolor{alphaCol1p0}{HTML}{21918C}
\definecolor{alphaCol2p0}{HTML}{28AE80}
\definecolor{alphaCol3p0}{HTML}{5EC962}
\definecolor{alphaCol5p0}{HTML}{ADDC30}
\definecolor{alphaCol10p0}{HTML}{FDE725}

\begin{tikzpicture}
\begin{groupplot}[
  group style={group size=2 by 1, horizontal sep=1.5cm},
  width=0.45\textwidth,
  height=0.35\textwidth,
  xmode=log,
  grid=major,
  grid style={dashed, gray!30},
  tick label style={font=\small},
  label style={font=\small},
  legend style={font=\scriptsize, draw=gray!60, fill=white, inner sep=2pt, row sep=-1pt},
  legend cell align={left},
]

\nextgroupplot[
  xlabel={Learning rate $\gamma$},
  ylabel={Test accuracy (\%)},
  xmin=0.01, xmax=10,
  ymin=48.9838, ymax=91.867,
]
\addplot+[alphaCol0p1, solid, no markers, thick] coordinates { (0.01,81.537) (0.025,86.25) (0.05,87.534) (0.075,87.011) (0.1,86.54) (0.2,84.55) (0.3,81.454) (0.4,76.355) (0.5,67.88) (0.75,62.051) (1,55.574) };
\addplot+[alphaCol0p2, solid, no markers, thick] coordinates { (0.01,74.243) (0.025,83.307) (0.05,86.784) (0.075,87.495) (0.1,87.645) (0.2,86.701) (0.3,86.844) (0.4,84.016) (0.5,82.933) (0.75,77.142) (1,70.761) };
\addplot+[alphaCol0p3, solid, no markers, thick] coordinates { (0.01,68.799) (0.025,80.084) (0.05,85.452) (0.075,86.757) (0.1,87.466) (0.2,88.119) (0.3,88.271) (0.4,88.114) (0.5,86.835) (0.75,84.367) (1,80.448) };
\addplot+[alphaCol0p5, solid, no markers, thick] coordinates { (0.01,63.904) (0.025,75.953) (0.05,83.16) (0.075,85.789) (0.1,86.737) (0.2,88.143) (0.3,88.598) (0.4,88.621) (0.5,87.818) (0.75,86.554) (1,86.13) };
\addplot+[alphaCol1p0, solid, no markers, thick] coordinates { (0.05,79.106) (0.075,83.491) (0.1,85.51) (0.15,87.04) (0.2,87.8044) (0.3,88.56) (0.5,89.137) (1,88.59) (1.5,87.178) (2,84.926) (3,81.5212) (5,70.75) (10,50.933) };
\addplot+[alphaCol2p0, solid, no markers, thick] coordinates { (0.05,75.419) (0.075,79.67) (0.1,82.588) (0.15,85.6625) (0.2,87.149) (0.3,88.033) (0.5,89.328) (1,88.818) (1.5,89.392) (2,88.499) (3,87.943) (5,78.929) (10,61.187) };
\addplot+[alphaCol3p0, solid, no markers, thick] coordinates { (0.05,74.095) (0.075,77.8875) (0.1,81.808) (0.15,85.52) (0.2,86.873) (0.3,88.174) (0.5,89.043) (1,88.962) (1.5,88.958) (2,88.85) (3,87.621) (5,81.0411) (10,63.606) };
\addplot+[alphaCol5p0, solid, no markers, thick] coordinates { (0.1,80.035) (0.2,86.077) (0.5,88.71) (0.75,89.358) (1,89.79) (2,89.9178) (3,88.545) (4,86.808) (5,84.814) (6,81.759) (7,78.393) };
\addplot+[alphaCol10p0, solid, no markers, thick] coordinates { (0.1,76.125) (0.2,84.6414) (0.5,88.0756) (0.75,88.852) (1,89.3889) (2,89.161) (3,88.842) (4,89.5871) (5,87.211) (6,87.143) (7,82.2212) };

\nextgroupplot[
  xlabel={Effective learning rate $\gamma_{\mathrm{eff}}=\gamma\cdot\kappa(z;r,\alpha)$},
  ylabel={Test accuracy (\%)},
  xmin=0.00169364, xmax=1.14762,
  ymin=48.9838, ymax=91.867,
  legend style={at={(1.03,1)}, anchor=north west,
                font=\scriptsize, draw=gray!60, fill=white,
                inner sep=2pt, row sep=-1pt},
]
\addplot+[alphaCol0p1, solid, no markers, thick] coordinates { (0.00684524,81.537) (0.0171131,86.25) (0.0342262,87.534) (0.0513393,87.011) (0.0684524,86.54) (0.136905,84.55) (0.205357,81.454) (0.27381,76.355) (0.342262,67.88) (0.513393,62.051) (0.684524,55.574) };
\addlegendentry{$\alpha=0.1$}
\addplot+[alphaCol0p2, solid, no markers, thick] coordinates { (0.00354433,74.243) (0.00886081,83.307) (0.0177216,86.784) (0.0265824,87.495) (0.0354433,87.645) (0.0708865,86.701) (0.10633,86.844) (0.141773,84.016) (0.177216,82.933) (0.265824,77.142) (0.354433,70.761) };
\addlegendentry{$\alpha=0.2$}
\addplot+[alphaCol0p3, solid, no markers, thick] coordinates { (0.00249636,68.799) (0.00624091,80.084) (0.0124818,85.452) (0.0187227,86.757) (0.0249636,87.466) (0.0499273,88.119) (0.0748909,88.271) (0.0998545,88.114) (0.124818,86.835) (0.187227,84.367) (0.249636,80.448) };
\addlegendentry{$\alpha=0.3$}
\addplot+[alphaCol0p5, solid, no markers, thick] coordinates { (0.00169364,63.904) (0.00423411,75.953) (0.00846821,83.16) (0.0127023,85.789) (0.0169364,86.737) (0.0338728,88.143) (0.0508093,88.598) (0.0677457,88.621) (0.0846821,87.818) (0.127023,86.554) (0.169364,86.13) };
\addlegendentry{$\alpha=0.5$}
\addplot+[alphaCol1p0, solid, no markers, thick] coordinates { (0.00573808,79.106) (0.00860712,83.491) (0.0114762,85.51) (0.0172142,87.04) (0.0229523,87.8044) (0.0344285,88.56) (0.0573808,89.137) (0.114762,88.59) (0.172142,87.178) (0.229523,84.926) (0.344285,81.5212) (0.573808,70.75) (1.14762,50.933) };
\addlegendentry{$\alpha=1$}
\addplot+[alphaCol2p0, solid, no markers, thick] coordinates { (0.00420963,75.419) (0.00631444,79.67) (0.00841925,82.588) (0.0126289,85.6625) (0.0168385,87.149) (0.0252578,88.033) (0.0420963,89.328) (0.0841925,88.818) (0.126289,89.392) (0.168385,88.499) (0.252578,87.943) (0.420963,78.929) (0.841925,61.187) };
\addlegendentry{$\alpha=2$}
\addplot+[alphaCol3p0, solid, no markers, thick] coordinates { (0.00385776,74.095) (0.00578664,77.8875) (0.00771553,81.808) (0.0115733,85.52) (0.0154311,86.873) (0.0231466,88.174) (0.0385776,89.043) (0.0771553,88.962) (0.115733,88.958) (0.154311,88.85) (0.231466,87.621) (0.385776,81.0411) (0.771553,63.606) };
\addlegendentry{$\alpha=3$}
\addplot+[alphaCol5p0, solid, no markers, thick] coordinates { (0.00646661,80.035) (0.0129332,86.077) (0.032333,88.71) (0.0484996,89.358) (0.0646661,89.79) (0.129332,89.9178) (0.193998,88.545) (0.258664,86.808) (0.32333,84.814) (0.387997,81.759) (0.452663,78.393) };
\addlegendentry{$\alpha=5$}
\addplot+[alphaCol10p0, solid, no markers, thick] coordinates { (0.0048199,76.125) (0.00963981,84.6414) (0.0240995,88.0756) (0.0361493,88.852) (0.048199,89.3889) (0.0963981,89.161) (0.144597,88.842) (0.192796,89.5871) (0.240995,87.211) (0.289194,87.143) (0.337393,82.2212) };
\addlegendentry{$\alpha=10$}

\end{groupplot}
\end{tikzpicture}
    \caption{
    Mean test accuracy of ResNet-18 on CIFAR10 under a fixed privacy budget $(\epsilon=1,\delta=10^{-5})$ for $\alpha\in[0.1, 10]$.        
    The left plot shows the mean test accuracy as a function of the learning rate $\gamma$.
    The right plot shows the same data plotted against the effective learning rate~$\gamma_{\mathrm{eff}}$
    from~\eqref{eq:JEFFISHERE}, 
    which produces overlapping curves    
    and shows that much of accuracy's dependence on $(\alpha,\gamma)$
    is explained by the update scale reparameterization due to~$\kappa(z;r,\alpha)$.
    }
    \label{fig:pick_alpha}
\end{figure}
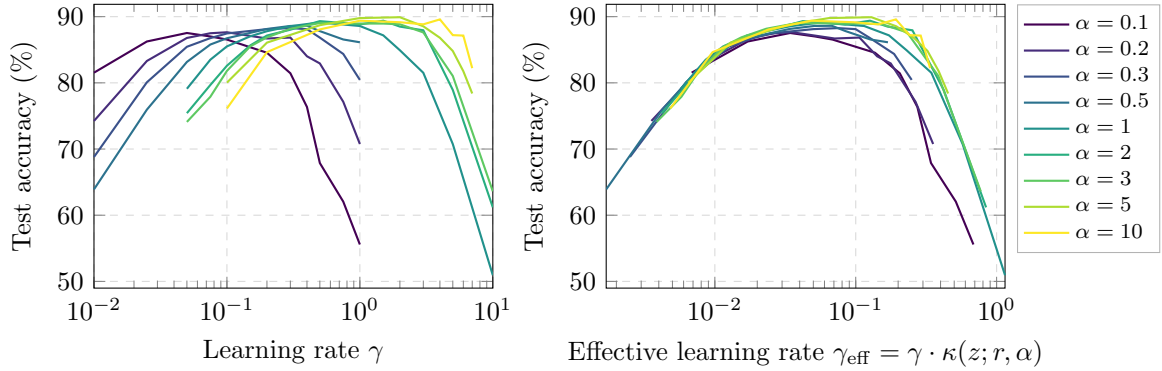

Figure~\ref{fig:pick_alpha} provides empirical results for the preceding analysis.
We consider $\alpha\in\{0.1, 0.2, \\ 0.3, 0.5,1,2,3,5,10\}$,
and for each $\alpha$ we first find $r^*$ using~\eqref{eq:optim} with privacy parameters $\epsilon=1$ and $\delta=10^{-5}$. 
For each $\alpha$ we train a neural network using varying values of $\gamma,$ specified in Appendix~\ref{apdnx:lr_values},
and for each value we train $10$ instances of ResNet-18 on CIFAR10 and report the mean test set accuracy in the left plot of Figure~\ref{fig:pick_alpha}.
Still using $s_y(z)=0.9$ for all examples, 
we then compute the test accuracy 
values for varying values of~$\alpha$ and~$\gamma$, and we
plot them against the effective learning rate $\gamma_{\textnormal{eff}}=\gamma \kappa(z;r,\alpha)$ on the right-hand side
of Figure~\ref{fig:pick_alpha}.
The different values of $\alpha$ 
produce curves that are quite similar, and 
this behavior shows that at a fixed privacy budget the dependence of performance on $(\alpha,\gamma)$ is 
largely 
explained by the change in the average update scale, namely $\kappa(z;r,\alpha)$.

In Appendix~\ref{ss:asymptotics}
we analyze the asymptotic behavior of $\kappa(z; r, \alpha)$
and $ \mathrm{Var}_{\tilde p}\bigl[\loss(\tilde p(z),y)\bigr]$ as functions of $\alpha$.
At small values of $\alpha,$ the value of $\kappa(z; r, \alpha)$
decays like $\alpha^{-1},$ while the variance decays like $\alpha^{-2}.$ At large $\alpha$, 
the expected update scale saturates, while the variance decays like $\alpha^{-1}.$
Thus, increasing $\alpha$ reduces the randomness
of the Dirichlet mechanism faster than it reduces the average update size, which allows for 
runs with larger values of~$\alpha$ to use larger learning rates~$\gamma$.

\section{Evaluation}
\label{sec:eval}

In this section we evaluate the proposed framework on standard benchmark data sets that have been 
used in previous works on private machine learning\footnote{Our code is available at \url{ https://github.com/Rhyme0730/Dirichlet-mechanism-for-private-classification}.}. 

\subsection{Experimental Setup}
\label{subsec:setup}
\subsubsection{Data sets and models} 
We evaluate Algorithm~\ref{alg:FDP-SGD} on CIFAR10~\citep{krizhevsky2009learning}, MNIST~\citep{lecun1998gradient}, DermaMNIST from MedMNIST~\citep{medmnistv2, medmnistv1}, FashionMNIST~\citep{xiao2017fashion}, and Street View House Numbers (SVHN)~\citep{netzer2011reading}. 
For CIFAR10, we use ResNet-18~\citep{he2016deep}, which was also used in~\citet{monir2024differentially}
when privatizing training inputs while leaving labels public.
For MNIST and FashionMNIST~\citep{xiao2017fashion}, we use a simplified Inception model following~\citet{szegedy2015going}.
For DermaMNIST, we use ResNet-9 following~\citet{holzl2022bridging}.
For SVHN, we use ResNet-18.
We provide results on CIFAR10, MNIST, DermaMNIST, FashionMNIST and SVHN in Section~\ref{sec:evaluation}.

\subsubsection{Implementation details}
For all data sets, we use the parameters $\alpha=3.0$ and $\delta=10^{-5}$. 
We train with standard SGD, 
%no momentum\footnote{\hr{We observe a little accuracy benefit from momentum in the private training of Algorithm~\ref{alg:FDP-SGD}, but its behavior becomes unstable when $\epsilon$ increases.}},
NLL loss, and $\ps$
for all data sets. 
We use learning rates $\gamma = 0.05$ for DermaMNIST and $\gamma = 0.1$ for all other data sets. 
% \hr{Note that epochs $E=qT$, we also report training epochs number of the following experiments.}
Algorithm~\ref{alg:FDP-SGD} is run for
$E=qT$ epochs on each data set, where~$q$ is the parameter used
in PoissonSample.
For CIFAR10, we apply standard data augmentations, including random cropping and random horizontal flipping, and training is done  
with sampling rate $q=250/50000 = 0.005$ for $T=20,000$ steps ($100$ epochs) to ensure a fair comparison with previous work in~\citep{monir2024differentially}.
On DermaMNIST, we use sampling rate $q=70/7007 \approx 0.01$ and train for $T=6006$ steps ($60$ epochs).
On MNIST and FashionMNIST, we use sampling rate $q=250/60000 \approx 0.0042$ and train for $T=9600$ steps ($40$ epochs).
On SVHN, we use sampling rate $q=250/73257 \approx 0.0034$ and train for $T=11,722$ steps ($40$ epochs).
The parameter~$r^*$ is computed for each data set 
using~\eqref{eq:optim}
based on the training configuration above and the target value of the privacy parameter~$\epsilon$.
On FashionMNIST and SVHN, we use Opacus~\citep{opacus} for our DP-SGD benchmark with maximum gradient clipping norm $1.0$.
We use the SGD optimizer with learning rate $0.1$ and no data augmentation with ResNet-18 as a baseline.
We report the value of $r^*$ we use in each experiment in Appendix~\ref{appendix:details}.
In the rest of this section, 
the mean accuracy and standard deviation for all data sets
are reported across 
$5$ independent runs 
with different random seeds for each value of~$\epsilon$.
We report results for different sampling rates in Appendix~\ref{appendix:diff_sample_rates} and for different optimizers in Appendix~\ref{appendix:diff_opt}.
All reported results are obtained from models trained from scratch\footnote{We run experiments for $\epsilon=\infty$ using $\ps$ but without privacy. The accuracy degradation from the non-private baseline at $\epsilon=\infty$ is caused by using $\ps$ and not adding momentum or weight decay to the SGD optimizer, which we choose to do in order to have a meaningful baseline for assessing the performance 
of Algorithm~\ref{alg:FDP-SGD}. Results marked with ``$\dagger$'' indicate non-private ($\epsilon=\infty$) baselines that were omitted from~\citep{monir2024differentially}. 
We reproduced these values using the configurations
described in the references corresponding to each method. 
}, not through fine-tuning existing models.
We use PyTorch~\citep{paszke2019pytorch} for our implementation and train all the models on Nvidia Quadro RTX6000 GPUs.

\subsection{Evaluation of the Proposed Framework}
\label{sec:evaluation}
We report the results of the proposed framework in Tables~\ref{tab:CIFAR10},~\ref{tab:MNIST}, and~\ref{tab:Derma_test} for the CIFAR10, MNIST, and DermaMNIST data sets, respectively.
We reiterate that 
DP-SGD privatizes both training inputs and labels, while 
the current paper 
and~\cite{monir2024differentially} only privatize training inputs. 
Therefore, the privacy protections of DP-SGD are strictly stronger than the protections provided
by Algorithm~\ref{alg:FDP-SGD}, though
we still compare to DP-SGD because it is widely used in private learning.

\subsubsection{CIFAR10 results} 
Table~\ref{tab:CIFAR10} compares Algorithm~\ref{alg:FDP-SGD}
to DP-SGD~\citep{de2022unlocking} using the best results 
that DP-SGD produces
on CIFAR10 from scratch with Wide-ResNet (WRN-40-4), large-batch training, and augmentation multiplicity. 
According to~\cite{zagoruyko2016wide}, the WRN-40-4 model obtains $>95\%$ accuracy on CIFAR10 when trained non-privately.
Algorithm~\ref{alg:FDP-SGD} under the configuration in Section~\ref{subsec:setup}
obtains higher performance than the WRN-40-4 configurations in~\citep{de2022unlocking}
for all values of~$\epsilon \in [1,8]$. 
Notably, when $\epsilon=1$, the test accuracy of DP-SGD is $56.4\%$, while the test accuracy of Algorithm~\ref{alg:FDP-SGD} is $82.96\%$, which is a $60.92\%$ reduction in test error rate. 

We also compare to the closest prior work, which is \citet{monir2024differentially}, using the same model (ResNet-18), the same optimizer, and the same sampling rate $q$ of $\ps$.
Since the learning rate $\gamma$ is not reported in~\cite{monir2024differentially}, we use $\gamma=0.1$ in training CIFAR10 for Algorithm~\ref{alg:FDP-SGD}.
% Specifically, Algorithm~\ref{alg:FDP-SGD} outperforms \citet{monir2024differentially}
% \mh{By ``outperforms'' do we mean ``has lower error rate than''?}
% by approximately $10$\% \mh{Should this~$10$\% number be higher?}
% when~$\epsilon=4$.
% Specifically, 
We see that Algorithm~\ref{alg:FDP-SGD} has a $45.3\%$ lower error rate than~\citet{monir2024differentially} when~$\epsilon=4$.
Since the accuracy at the strongest privacy level for Algorithm~\ref{alg:FDP-SGD} ($\epsilon=1$) already exceeds the reported accuracy at the weakest privacy level in~\cite{monir2024differentially} ($\epsilon=4.46$), we omit the accuracy of Algorithm~\ref{alg:FDP-SGD} at $\epsilon=4.46$
from Table~\ref{tab:CIFAR10}.

\begin{table}[t]
  \centering
  \small
  \begin{tabular}{@{}l c *{7}{c}@{}}
    \toprule
    Method & Model & $\epsilon=1$ & $\epsilon=2$ & $\epsilon=4$ & $\epsilon=6$ & $\epsilon=8$ & $\epsilon=\infty$ \\
    \midrule
    % DP-SGD~\citep{de2022unlocking}      & WRN-16-4  & $56.8$  & $64.9$  & $69.2$  & $71.9$  & $77.0$  & $79.5$  &              \\
    DP-SGD~\citep{de2022unlocking}      & WRN-40-4  & $56.4$  & $65.9$  & $73.5$  & $78.8$  & $81.4$  & ---            \\
    \citet{monir2024differentially} & ResNet-18 & --- & --- & $78.37$ & $79.0_{(\epsilon=4.46)}$  & --- & $92.35^{\dagger}$ \\
    \midrule
    \multirow{2}{*}{This work} & \multirow{2}{*}{ResNet-18} & $82.96$    & $86.34$       & $88.17$    & $88.54$    & $88.61$    & $92.35$ \\
                          &                            & $\pm 1.46$ & $\pm 0.34$  & $\pm 0.33$ & $\pm 0.54$ & $\pm 0.34$ &         \\
    \bottomrule
  \end{tabular}
  \caption{Test accuracy (\%) on CIFAR10 under $\delta=10^{-5}$ across 5 independent runs for each~$\epsilon$. 
  }
  \label{tab:CIFAR10}
\end{table}

\subsubsection{MNIST results} 

\begin{table}[b]
  \centering
  \small
  \begin{tabular}{@{}l c *{6}{c}@{}}
    \toprule
    Method & Model & $\epsilon=0.5$ & $\epsilon=1$ & $\epsilon=2$ & $\epsilon=4$  & $\epsilon=\infty$ \\
    \midrule
    DP-SGD                & MLP       & $90.0$    & ---       & $95.0$    & ---        & $98.3$            \\
    \cite{papernot2021tempered} & CNN       & ---       & ---       & --- & \hspace{-1.5em}$98.1_{(\epsilon=2.93)}$\hspace{1.5em}            & $99.0$  \\
    \citet{cheng2022dpnas}                 & DPNASNet   & ---       & ---       & ---       & \hspace{-1.5em}$98.57_{(\epsilon=3)}$\hspace{1.5em} & ---      \\
    \citet{monir2024differentially} & CNN       & ---       & $98.0$    & ---       & ---           & $99.4^{\dagger}$  \\
    \midrule
    \multirow{2}{*}{This work} & \multirow{2}{*}{Inception} & $99.28$    & $99.37$    & $99.40$    & $99.44$      & $99.51$ \\
                          &                            & $\pm 0.072$ & $\pm 0.059$ & $\pm 0.058$ & $\pm 0.038$ &         &     \\
    \bottomrule
  \end{tabular}
  \caption{Test accuracy (\%) on MNIST under $\delta=10^{-5}$ across 5 independent runs for each value of~$\epsilon$.}
  \label{tab:MNIST}
\end{table}

As shown in Table~\ref{tab:MNIST}, Algorithm~\ref{alg:FDP-SGD} has higher accuracy than all DP-SGD baselines and 
the method of \citet{monir2024differentially} across all tested values of $\epsilon$.
We applied
Algorithm~\ref{alg:FDP-SGD} to
a simplified Inception model~\citep{szegedy2015going}, 
which is a type of convolutional neural network (CNN). 
%for our evaluation because its non-private accuracy is close to that of other works. In particular, 
In Table~\ref{tab:MNIST}, 
we see that the accuracy of the Inception model at $\epsilon=0.5$ closely matches the non-private accuracy of the Inception model,
demonstrating that Algorithm~\ref{alg:FDP-SGD} preserves high utility even under strong privacy protections. 
Note that \citet{papernot2021tempered} and \citet{cheng2022dpnas} only report results at $\epsilon=2.93$ and $\epsilon=3$, respectively,
and we compare the accuracy of those methods to the accuracy of Algorithm~\ref{alg:FDP-SGD} at nearby~$\epsilon$ values. 
We also compare to DP-SGD~\citep{abadi2016deep} which uses a multilayer perceptron (MLP) model.
Even at the strongest privacy level of $\epsilon=0.5$, Algorithm~\ref{alg:FDP-SGD} produces higher accuracy than \citet{monir2024differentially} at $\epsilon=1$.
Also, the method in \citet{monir2024differentially} requires a specific funnel-shaped CNN designed to reduce the noise injected during training,
while Algorithm~\ref{alg:FDP-SGD} can be applied to any neural network architecture that has a softmax output layer. 
The improvement in test accuracy of Algorithm~\ref{alg:FDP-SGD} over \citet{monir2024differentially}
is due to two factors: (i) we inject randomness only at the \softmax layer rather than at layers throughout the network, and (ii) under the same privacy budget, the Dirichlet mechanism introduces randomness with lower variance than the Gaussian mechanism from~\citet{monir2024differentially}.

\subsubsection{DermaMNIST results}
We have shown that Algorithm~\ref{alg:FDP-SGD} outperforms previous work on the standard computer vision benchmarks of CIFAR10 and MNIST, and we now extend our evaluation to medical images, where the privacy of input data is critical.
For training, we applied Algorithm~\ref{alg:FDP-SGD} to ResNet-9, which was also used in~\citep{holzl2022bridging}.
There is limited previous work that evaluates the accuracy of private machine learning models on sensitive medical data such as DermaMNIST~\citep{medmnistv1, medmnistv2}.
Evaluating privately learned models on this data set directly highlights the practical significance of Algorithm~\ref{alg:FDP-SGD}, which privatizes inputs that are medical images to ensure privacy for patients.
In Table~\ref{tab:Derma_test} we see that Algorithm~\ref{alg:FDP-SGD} outperforms the benchmark DP-SGD method~\citep{tang2023differentially} from $\epsilon=1$ to $\epsilon=7.42$ when using the same architecture\footnote{The accuracy gap at $\epsilon=\infty$ between DP-SGD and Algorithm~\ref{alg:FDP-SGD} arises from DP-SGD's use of group normalization (required for DP compatibility) versus batch normalization used in our work.}.
Furthermore, the gap between our accuracy at $\epsilon=1$ and the non-private baseline is less than $5\%$, demonstrating that Algorithm~\ref{alg:FDP-SGD} provides strong privacy protections while maintaining high model accuracy.
Due to DermaMNIST's train/validation/test partitioning, we also report the validation accuracy in Appendix~\ref{appendix:Derma}.

\begin{table}[t]
  \centering
  \begin{tabular}{@{}l c *{4}{c}@{}}
    \toprule
    Method & Model & $\epsilon=1$ & $\epsilon=4$ & $\epsilon=7.42$ & $\epsilon=\infty$ \\
    \midrule
    DP-SGD~\citep{tang2023differentially} & ResNet-9 & $68.34$ & $71.08$ & $72.58$ & $76.16$ \\
    \midrule
    \multirow{2}{*}{This work} & \multirow{2}{*}{ResNet-9} & $74.32$    & $75.74$    & $76.31$    & $78.60$ \\
                          &                           & $\pm 1.84$ & $\pm 0.74$ & $\pm 1.15$ &     \\
    \bottomrule
  \end{tabular}
  \caption{Test accuracy (\%) on DermaMNIST under $\delta=10^{-5}$ across 5 independent runs for each $\epsilon$.}
  \label{tab:Derma_test}
\end{table}

\subsubsection{FashionMNIST and SVHN results}
% Requires in your preamble:
%   \usepackage{pgfplots}
%   \usepackage{subcaption}   % provides the subfigure environment + \caption
%   \pgfplotsset{compat=1.18} % (or your version)

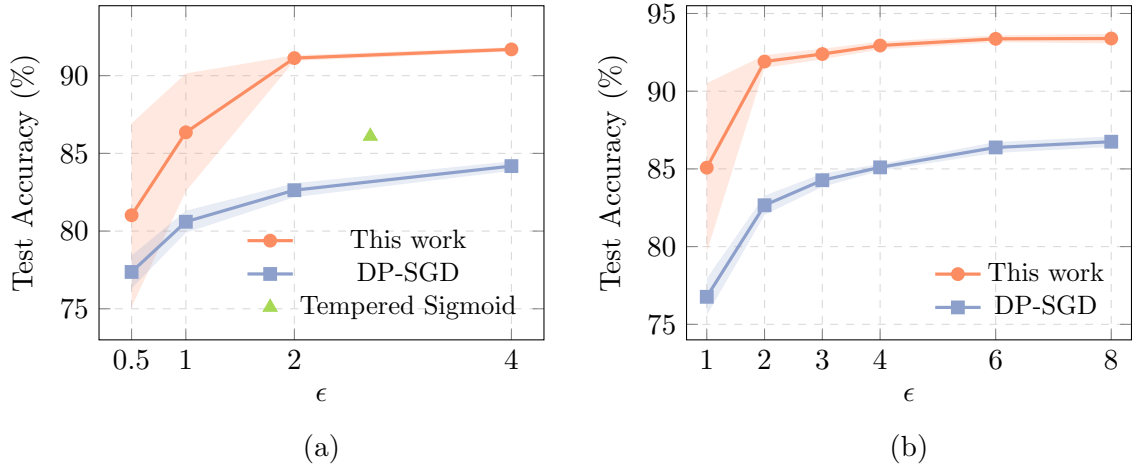
\begin{figure}[t]
  \centering
  % --- Shared colors (declared once for both subfigures) ---
  \providecolor{cbTeal}  {HTML}{66C2A5}
  \providecolor{cbOrange}{HTML}{FC8D62}
  \providecolor{cbPurple}{HTML}{8DA0CB}
  \providecolor{cbPink}  {HTML}{E78AC3}
  \providecolor{cbGreen} {HTML}{A6D854}
  \providecolor{cbYellow}{HTML}{FFD92F}
  \providecolor{cbGrey}  {HTML}{B3B3B3}

  %==================== Subfigure (a): Inception ====================
  \begin{subfigure}[t]{0.49\linewidth}
    \centering
    \begin{tikzpicture}
    \begin{axis}[
      width=\linewidth,
      height=6.0cm,            % reduce if the side-by-side plots look too tall
      grid=both,
      grid style={dashed, gray!30},
      xlabel={$\epsilon$},
      ylabel={Test Accuracy (\%)},
      xmin=0.2, xmax=4.3,
      ymin=73.0, ymax=94.5,
      xtick={0.5,1,2,4},
      legend pos=south east,
      legend style={font=\small, draw=none, fill=none}
    ]
    % --- Shaded CI (Ours) ---
    \path [draw=none, fill=cbOrange, opacity=0.2]
    (axis cs:0.5,75.11) -- (axis cs:1,82.56) -- (axis cs:2,90.95) -- (axis cs:4,91.55)
    -- (axis cs:4,91.83) -- (axis cs:2,91.31) -- (axis cs:1,90.14) -- (axis cs:0.5,86.93)
    -- cycle;
    % --- Shaded CI (DP-SGD) ---
    \path [draw=none, fill=cbPurple, opacity=0.2]
    (axis cs:0.5,76.31) -- (axis cs:1,79.89) -- (axis cs:2,82.18) -- (axis cs:4,83.87)
    -- (axis cs:4,84.49) -- (axis cs:2,83.08) -- (axis cs:1,81.31) -- (axis cs:0.5,78.43)
    -- cycle;
    % --- Main line (Ours) ---
    \addplot [line width=1.2pt, cbOrange, mark=*]
    table {%
    0.5 81.02
    1   86.35
    2   91.13
    4   91.69
    };
    \addlegendentry{This work}
    % --- Main line (DP-SGD) ---
    \addplot [line width=1.2pt, cbPurple, mark=square*]
    table {%
    0.5 77.37
    1   80.60
    2   82.63
    4   84.18
    };
    \addlegendentry{DP-SGD}
    % --- Tempered Sigmoid (single point, eps = 2.7) ---
    \addplot [only marks, cbGreen, mark=triangle*, mark size=3pt]
    coordinates {(2.7,86.1)};
    \addlegendentry{Tempered Sigmoid}
    \coordinate (ttl) at (axis description cs:0.5,-0.4);
        \end{axis}
        \node[anchor=south,text width=7cm,align=center] at (ttl) 
{(a)};
    \end{tikzpicture}
    \label{fig:dp-inception}
  \end{subfigure}
  \hfill
  %==================== Subfigure (b): ResNet-18 ====================
  \begin{subfigure}[t]{0.49\linewidth}
    \centering
    \begin{tikzpicture}
    \begin{axis}[
      width=\linewidth,
      height=6.0cm,
      grid=both,
      grid style={dashed, gray!30},
      xlabel={$\epsilon$},
      ylabel={Test Accuracy (\%)},
      xmin=0.65, xmax=8.35,
      ymin=74.0, ymax=95.5,
      xtick={1,2,3,4,6,8},
      legend pos=south east,
      legend style={font=\small, draw=none, fill=none}
    ]
    % --- Shaded CI (Ours) ---
    \path [draw=none, fill=cbOrange, opacity=0.2]
    (axis cs:1,90.51) -- (axis cs:1,79.65)
    -- (axis cs:2,91.50) -- (axis cs:3,92.05)
    -- (axis cs:4,92.69) -- (axis cs:6,93.15)
    -- (axis cs:8,93.08) -- (axis cs:8,93.70)
    -- (axis cs:6,93.59) -- (axis cs:4,93.19)
    -- (axis cs:3,92.73) -- (axis cs:2,92.32)
    -- (axis cs:1,90.51) -- cycle;
    % --- Shaded CI (DP-SGD) ---
    \path [draw=none, fill=cbPurple, opacity=0.2]
    (axis cs:1,77.96) -- (axis cs:1,75.58)
    -- (axis cs:2,82.04) -- (axis cs:3,83.84)
    -- (axis cs:4,84.88) -- (axis cs:6,86.03)
    -- (axis cs:8,86.41) -- (axis cs:8,87.09)
    -- (axis cs:6,86.73) -- (axis cs:4,85.32)
    -- (axis cs:3,84.70) -- (axis cs:2,83.28)
    -- (axis cs:1,77.96) -- cycle;
    % --- Main line (Ours) ---
    \addplot [line width=1.2pt, cbOrange, mark=*]
    table {%
    1 85.08
    2 91.91
    3 92.39
    4 92.94
    6 93.37
    8 93.39
    };
    \addlegendentry{This work}
    % --- Main line (DP-SGD) ---
    \addplot [line width=1.2pt, cbPurple, mark=square*]
    table {%
    1 76.77
    2 82.66
    3 84.27
    4 85.10
    6 86.38
    8 86.75
    };
    \addlegendentry{DP-SGD}
    \coordinate (ttl) at (axis description cs:0.5,-0.4);
        \end{axis}
        \node[anchor=south,text width=7cm,align=center] at (ttl) 
{(b)};
    \end{tikzpicture}
    \label{fig:dp-resnet}
  \end{subfigure}

  \caption{Test accuracy (\%) under $\delta=10^{-5}$ across 5 independent runs for each $\epsilon$ on (a) FashionMNIST 
  and (b) SVHN. Shaded bands show one standard deviation for each $\epsilon$.}
  \label{fig:dp-accuracy}
\end{figure}
As shown in Figure~\ref{fig:dp-accuracy}, Algorithm~\ref{alg:FDP-SGD} has higher accuracy than DP-SGD across all tested values of~$\epsilon$.
Notably, on FashionMNIST at~$\epsilon=4.0$, Algorithm~\ref{alg:FDP-SGD} achieves a $48.93\%$ reduction in test error rate relative to DP-SGD. 
Algorithm~\ref{alg:FDP-SGD} also outperforms Tempered Sigmoid~\citep{papernot2021tempered}, which reports the highest accuracy among existing methods for private learning.
On SVHN at~$\epsilon=1.0$, Algorithm~\ref{alg:FDP-SGD} achieves a $35.77\%$ reduction in test error rate
relative to DP-SGD.
The full results are shown in Table~\ref{tab:FashionMNIST} for FashionMNIST and Table~\ref{tab:SVHN} for SVHN.

\begin{table}[!t]
  \centering
  \small
  \begin{tabular}{@{}l *{7}{c}@{}}
    \toprule
    Method 
    & Model 
      & $\epsilon = 0.5$
      & $\epsilon = 1$
      & $\epsilon = 2$
      & $\epsilon = 4$
      & $\epsilon = \infty$ \\
    \midrule
    Tempered Sigmoid & CNN & --- & --- & $86.1_{(\epsilon=2.7)}$ & --- & $89.4$ \\
    \multirow{2}{*}{DP-SGD}
      & \multirow{2}{*}{Inception}
      & $77.37$
      & $80.60$
      & $82.63$
      & $84.18$
      & $93.13$
      &  \\
      & 
      & $\pm 1.06$
      & $\pm 0.71$
      & $\pm 0.45$
      & $\pm 0.31$
      & 
      &  \\
    \midrule
    \multirow{2}{*}{This work}
      & \multirow{2}{*}{Inception}
      & $81.02$
      & $86.35$
      & $91.13$
      & $91.69$
      & $93.13$
      &  \\
      & 
      & $\pm5.91$
      & $\pm3.79$
      & $\pm0.18$
      & $\pm0.14$
      &  \\
    \bottomrule
  \end{tabular}
  \caption{Test accuracy (\%) on FashionMNIST under $\delta=10^{-5}$ across 5 independent runs for each $\epsilon$.}
  \label{tab:FashionMNIST}
\end{table}

\begin{table}[t]
  \centering
  \small
  \begin{tabular}{@{}l c *{7}{c}@{}}
    \toprule
    Method & Model & $\epsilon{=}1$ & $\epsilon{=}2$ & $\epsilon{=}3$ & $\epsilon{=}4$ & $\epsilon{=}6$ & $\epsilon{=}8$ & $\epsilon{=}\infty$ \\
    \midrule
    \multirow{2}{*}{DP-SGD} & \multirow{2}{*}{ResNet-18} & $76.77$ & $82.66$ & $84.27$ & $85.10$ & $86.38$ & $86.75$ & $94.58$ \\
                            &  & $\pm1.19$ & $\pm0.62$ & $\pm0.43$ & $\pm0.22$ & $\pm0.35$ & $\pm0.34$ & \\
    \midrule
    \multirow{2}{*}{This work}   & \multirow{2}{*}{ResNet-18} & $85.08$ & $91.91$ & $92.39$ & $92.94$ & $93.37$ & $93.39$ & $94.58$ \\
                            &  & $\pm5.43$ & $\pm0.41$ & $\pm0.34$ & $\pm0.25$ & $\pm0.22$ & $\pm0.31$ & \\
    \bottomrule
  \end{tabular}
  \caption{Test accuracy (\%) on SVHN under $\delta=10^{-5}$ across 5 independent runs for each $\epsilon$.}
  \label{tab:SVHN}
\end{table}

\subsubsection{Computational cost}
Our private training framework can be split into two stages. 
First, we compute $r^*$ 
using~\eqref{eq:optim} for a fixed DP budget $(\epsilon,\delta)$ and a fixed value of $\alpha$, which usually takes less than one minute for each $\epsilon\in[1,8]$.
Once $r^*$ is computed, we follow the traditional image training process with an additional step 
of perturbing output probabilities 
using the Dirichlet mechanism
at each training step, 
for which the additional time complexity is $O(dq|\D|T)$ for the whole training process, where 
$q|\D|$ is the expected mini-batch size
at each iteration of Algorithm~\ref{alg:FDP-SGD}
and $d$ is the number of classes.
Compared to previous work~\citep{monir2024differentially} that follows the DP-SGD approach for adding noise to model layer outputs and using per-sample clipping, 
the training process of Algorithm~\ref{alg:FDP-SGD} is significantly faster.
When privately training ResNet-18 on CIFAR10, training an epoch ($200$ steps) for Algorithm~\ref{alg:FDP-SGD} consumes $\sim30$ seconds, while for DP-SGD it consumes $\sim60$ seconds.
The main reason Algorithm~\ref{alg:FDP-SGD}'s training speed is faster is because per-sample clipping is unnecessary for it.

\section{Discussion}
\citet{monir2024differentially} show that injecting calibrated Gaussian noise into the penultimate layer of a neural network can privatize inputs by randomizing features while keeping labels public, 
and that work demonstrates that this approach 
to privacy 
produces higher accuracy than DP-SGD.
However, the variance of privacy noise in \citet{monir2024differentially} scales with the width of the penultimate layer, which degrades utility for larger models.
In contrast, the variance of the Dirichlet mechanism scales with the number of classes in the data set, yielding substantially better utility under the same privacy guarantee.

A second key difference is that Algorithm~\ref{alg:FDP-SGD} eliminates the need for per-sample gradient clipping, which~\citet{monir2024differentially} use. 
Clipping degrades utility and slows training, since each sample's gradient must be clipped individually.
Comparing the improvements between Algorithm~\ref{alg:FDP-SGD} and~\citet{monir2024differentially} shows two contributions of our work: 
(i) the use of the Dirichlet mechanism instead of the Gaussian mechanism and (ii) the choice of layer at which privacy is introduced.
At $\epsilon=4.0$, we observe a $45.31\%$ lower error rate
in test accuracy on CIFAR10 
relative to the method in~\citet{monir2024differentially}, while
using the same architecture (ResNet-18).
And the accuracy of Algorithm~\ref{alg:FDP-SGD}
at the strongest privacy level we test ($\epsilon=1.0$) exceeds the accuracy 
of the method of~\citet{monir2024differentially}
at the weakest privacy level they test ($\epsilon=4.46$).

While DP-SGD~\citep{abadi2016deep} provides a stronger privacy guarantee since it privatizes both training inputs and labels, it has worse performance compared to Algorithm~\ref{alg:FDP-SGD}.
Algorithm~\ref{alg:FDP-SGD} only provides privacy to training inputs,
which allows it to introduce less randomness
into training and produce models with higher accuracy. 
This difference allows users to choose which part of their data requires privacy protections, 
and Algorithm~\ref{alg:FDP-SGD} allows users to avoid
unnecessary accuracy loss from privatizing labels that do not need privacy protections.
Algorithm~\ref{alg:FDP-SGD} is therefore complementary to DP-SGD because it provides different
privacy protections and serves different use cases than DP-SGD. 
%In terms of private machine learning, the first quadrant is non-private training. The second is DP-SGD~\citep{abadi2016deep}, which privatizes both training inputs and labels. The third is Label-DP~\citep{ghazi2021deep}, which privatizes only labels and leaves training inputs public. Together with our guidance on the corresponding learning rate to choose, these contributions complete the final quadrant \hr{final component} of the private machine learning landscape. \mh{The word ``quadrant'' is weird here. Can we change it to something else?}

\section{Conclusion}
In this work, we introduced a novel private training procedure that privatizes input data while keeping labels public when training deep neural networks. 
We showed that the Dirichlet mechanism for differential privacy can be implemented at a \softmax layer, with the resulting randomized probabilities 
used for computing the loss and backpropagation. 
The overall framework can be broken into two stages: (i) using an upper bound on the privacy budget to determine the parameters for the Dirichlet mechanism that produce the desired variance of privatized data,
 and (ii) implementing the Dirichlet mechanism
at the softmax layer at each training step.
Empirical evaluations on CIFAR10, MNIST, DermaMNIST, Fashion MNIST and SVHN
demonstrate that our framework substantially outperforms the prior state of
the art by improving test accuracy from $78.37\%$ to $88.17\%$ at $\epsilon=4$, and retaining meaningful utility even 
under strong privacy at $\epsilon=1$. 
Furthermore, 
the improvements of 
our approach persist across different image data sets, model architectures, and optimizers.
Future work includes the development of additional parameter selection rules to boost accuracy, as well as extending the 
private training framework to other machine learning tasks such as language, speech, and financial modeling.

\acks{This work was supported by 
NSF under CAREER grant 2422260 and Graduate Research Fellowship grant DGE-2039655,
AFOSR under grant FA9550-19-1-0169, and
ONR under grant N00014-24-1-2432. 
The authors declare no competing interests.
%Partnership for an Advanced Computing Environment (PACE) at the Georgia Institute of Technology, Atlanta, Georgia, USA (RRID:SCR\_027619)
}

\newpage
\appendix
\section{Proof of Lemma~\ref{lem:sensitivities}}
\label{appendix:proof_sen}
For any inputs $x, x' \in \X$ and any adjacent data sets~$\D$ and~$\D'$,
both $g_\theta(x)$ and $g_\theta(x')$ lie in the simplex $S^{d-1}$ because 
    a \softmax layer always outputs a probability vector. We define $a = g_{\theta}(x)$ and $b = g_{\theta}(x')$,
    and for~$k \in [d]$ 
    we use $a_k, b_k \geq 0$ to denote the $k^{th}$ entries of $a$ and $b$, respectively.     
    Then 
    %\begin{equation}
        $\sum_{k=1}^d a_k = \sum_{k=1}^d b_k = 1$.
    %\end{equation}
    For the $\ell^2$-sensitivity, we have
    \begin{equation}
        \|a-b\|^2_2 = \sum_{k=1}^d \big(a_k-b_k\big)^2 = \sum_{k=1}^d a_k^2 + \sum_{k=1}^d b_k^2 - 2\sum_{k=1}^d a_kb_k.    
    \end{equation}
    Then, since $a_k,b_k \in [0,1]$ we have $\sum_{k=1}^d a_k^2 \leq \sum_{k=1}^d a_k = 1$, 
    along with 
    $\sum_{k=1}^d b_k^2 \leq \sum_{k=1}^d b_k = 1$, 
    and $\sum_{k=1}^d a_kb_k \geq 0$. Hence
    \begin{equation}
        \Delta_2^2 = \sup_{x,x' \in \X} \|g_\theta(x) - g_\theta(x')\|^2_2 \leq 1+1-0 \leq 2.
    \end{equation}
    For the $\ell_\infty$-sensitivity, we have
    \begin{equation}
        \|a - b\|_{\infty} \leq \max_{k \in [d]} \{a_k, b_k\} \leq 1,
    \end{equation}
    and therefore
    \begin{equation}
        \Delta_\infty = \sup_{x, x' \in \X} \|g_{\theta}(x) - g_{\theta}(x')\|_\infty \leq 1.
    \end{equation}

\section{Proof of Theorem~\ref{theorem:MA}}
\label{appendix:proof_1}

We use the lemmas below in proving Theorem~\ref{theorem:MA}. 

\subsection{Supporting Lemmas}

% \hr{It looks like this lemma is already included in the Section 3.1, do we still show it here?}
% \begin{lemma}[{\citep[Appendix E]{ponnoprat2021dirichlet}}]
% \label{thm:dir_rdp}
% Fix $\lambda \geq 1$ and $r,\alpha >0.$ If $\lambda <1+\frac{\alpha}{r}$, 
% then Mechanism~\ref{mech:Dirichlet} is $(\lambda, \hat{\epsilon}(\lambda;r))$-Renyi 
% differentially private with
% \begin{equation}
%      \hat{\epsilon}(\lambda;r)=\lambda r^2\psi'(\alpha-(\lambda-1)r).
% \end{equation}
% \end{lemma}

\begin{lemma}[Composition of RDP mechanisms,~\citealt{mironov2017renyi}]
\label{lemma:RDP_composition}
    % Let $\M=(\M_1,\dots, \M_k)$, where, 
    % for all $i\in[k]$, 
    % $\M_i$ can depend on the outputs of $\M_1,\dots,\M_{i-1}$. 
    % Then $\M$ is Renyi-differentially private with $\hat{\epsilon}_\M(\cdot)=\sum_{i=1}^{k} \hat{\epsilon}_{\M_i}(\cdot)$.    
    Fix $\lambda\geq1$ and $k\in\mathbb{N}$ privacy mechanisms $\M_1,\dots, \M_k$, where mechanism $i$ is $(\lambda,\hat\epsilon_{\M_i}(\lambda))-$RDP for all $i\in[k].$
     Let $\M=(\M_1,\dots, \M_k)$, where, 
    for all $i \in [k] \backslash \{1\}$, 
    $\M_i$ can depend on the outputs of $\M_1,\dots,\M_{i-1}$. 
    Then $\M$ is $(\lambda,\hat\epsilon_\mathcal{M}(\lambda))-$RDP with $\hat{\epsilon}_\M(\lambda)=\sum_{i=1}^{k} \hat{\epsilon}_{\M_i}(\lambda)$.   
\end{lemma}

\begin{lemma}[RDP to DP conversion,~\citealt{canonne2020discrete}]
\label{lemma:RDP2DP}
    Let $\epsilon, \delta > 0$. If $\M:\DD \to \Simplex$ is $(\lambda, \hat{\epsilon})$-Renyi differentially private in the sense
    of Definition~\ref{def:RDP}, then it is $(\epsilon,\delta)$-differentially private in the sense of Definition~\ref{def:dp}, with
    \begin{equation}
        \epsilon = \hat{\epsilon} + \log(\lambda-1) - \frac{\log(\delta)+\lambda \log(\lambda)}{\lambda-1}.
    \end{equation}
\end{lemma}

\subsection{Proof}
Each iteration of Algorithm~\ref{alg:FDP-SGD} applies Mechanism~\ref{mech:Dirichlet} to a Poisson subsampled minibatch of~$\D$. By Lemma~\ref{lem:mech1rdp}, 
for any $\lambda \in [1, 1 +\frac{\alpha}{r})$
a single application of the Dirichlet mechanism enforces $(\lambda,\hat{\epsilon}_{\mathcal{M}}(\lambda;r))$-Renyi differential privacy with
$\hat{\epsilon}_{\mathcal{M}}(\lambda;r)=\lambda r^2\psi'(\alpha-(\lambda-1)r)$. 
In~\citep{ponnoprat2021dirichlet}, $\alpha$ is set to be a function 
of~$r$ and the RDP order $\lambda$.
In our analysis, we adapt the moments accountant from~\citep{zhu2019poission}, in which the underlying mechanism must be fixed while its RDP protections are evaluated across different values of~$\lambda$ and~$\hat{\epsilon}_{\mathcal{M}}(\lambda;r)$.
Therefore, in our training procedure we fix the Dirichlet mechanism parameters $\alpha$ and $r$ independently of $\lambda$.

Since each minibatch is drawn using Poisson subsampling with rate $q,$ the per-step 
RDP protections are amplified.
It follows from \citep[Theorem 5]{zhu2019poission} that for any integer value of~$\lambda \in [2, 1+\frac{\alpha}{r}),$ 
we have
\begin{equation}
\begin{aligned}
        \hat\epsilon_{sub}(\lambda;r)
&\leq \frac{1}{\lambda-1} \log \biggl \{ (1-q)^{\lambda-1}(q\lambda-q+1) \\ 
    &+ \binom{\lambda}{2} q^2 (1-q)^{\lambda-2} e^{\hat\epsilon_\M(2;r)} + 3\sum^\lambda_{j=3}\binom{\lambda}{j}(1-q)^{\lambda-j}q^je^{(j-1)\hat\epsilon_\M(j;r)} \biggr \},
\end{aligned}
\end{equation}
where~$\big(\lambda, \hat{\epsilon}_{sub}(\lambda; r)\big)$ is the 
strength of RDP protection 
provided to the minibatch
at a single training step. 

Over $T$ training steps, the overall RDP guarantees compose according to Lemma~\ref{lemma:RDP_composition}, and the training data is protected with~$\big(\lambda, \hat{\epsilon}_o(\lambda; r)\big)$-RDP, where 
\begin{equation}
\begin{aligned}
        \hat\epsilon_o(\lambda;r)
&\leq \frac{T}{\lambda-1} \log \biggl \{ (1-q)^{\lambda-1}(q\lambda-q+1) \\ 
    &+ \binom{\lambda}{2} q^2 (1-q)^{\lambda-2} e^{\hat\epsilon_\M(2;r)} + 3\sum^\lambda_{j=3}\binom{\lambda}{j}(1-q)^{\lambda-j}q^je^{(j-1)\hat\epsilon_\M(j;r)} \biggr \}.
\end{aligned}
\end{equation}

Lastly, using Lemma~\ref{lemma:RDP2DP} for a fixed $\delta>0$  gives
\begin{equation}
\begin{aligned}
    {\epsilon}_o(\lambda,\delta;r) &\leq \frac{T}{\lambda-1} \log \biggl \{ (1-q)^{\lambda-1}(q\lambda-q+1)   + \binom{\lambda}{2} q^2 (1-q)^{\lambda-2} e^{\hat\epsilon_\M(2;r)} \\
    & + 3\sum^\lambda_{j=3}\binom{\lambda}{j}(1-q)^{\lambda-j}q^je^{(j-1)\hat\epsilon_\M(j;r)} \biggr \} + \log(\lambda-1) - \frac{\log(\delta)+\lambda \log(\lambda)}{\lambda-1}
\end{aligned}
\end{equation}
for any integer value of~$\lambda \in [2, 1+\frac{\alpha}{r}).$

\section{Proof of Theorem~\ref{thm:main}} \label{proof:thm2}
We first prove
\begin{equation}
    F(z)=\E_{\tilde p}\bigl[\loss(\tilde p(z), y)\bigr]=\psi(r + d\alpha) - \psi\bigl(r s_y(z) + \alpha\bigr)\label{eq:prove_me_1}
\end{equation} and
\begin{equation}
    \mathrm{Var}_{\tilde p}\bigl[\loss(\tilde p(z), y)\bigr]
    =  \psi'\bigl(r s_y(z) + \alpha\bigr) - \psi'(r + d\alpha).\label{eq:prove_me_2}
    \end{equation}

The derivation of these equations uses properties of the exponential family of distributions:
\begin{definition}[\citealt{wainwright2008graphical}, Proposition 3.1]\label{def:exp_family}
The exponential family of distributions contains densities of the form
\begin{equation}
    p(x;\vartheta)=\exp\left( \sum_{i=1}^d\vartheta_i\phi_i(x)-A(\vartheta)\right)h(x),
\end{equation}
where $x\in\mathbb{R}^n$ denotes a generic realization of the random variable being modeled and $\vartheta\in\mathbb{R}^d$ is the parameter indexing the distribution.
For~$i, j \in [d]$ these distributions satisfy 
\begin{align}
    \frac{\partial A(\vartheta)}{\partial\vartheta_i}&=\E[\phi_i(X)] \label{eq:exp_id_1}\\
    \frac{\partial^2 A(\vartheta)}{\partial\vartheta_i\partial \vartheta_j} &=\textnormal{Cov}(\phi_i(X),\phi_j(X)).\label{eq:exp_id_2}
\end{align}
\end{definition}

Let $z \in \R^d$ be the network logits for a single example~$x$ with true label
$y \in [d]$.
Let $\eta(z) = r s(z) + \alpha\onevec_d$,
with
\begin{equation} \label{eq:etayandtau}
  \eta_y(z) = r s_y(z) + \alpha,
  \qquad
  \tau = \sum_{j=1}^d \eta_j(z) = r + d\alpha.
\end{equation}
The probability density function used in Mechanism~\ref{mech:Dirichlet} takes the form 
\[
  f(\tilde p; \eta(z))
  = \frac{1}{B(\eta(z))}
        \prod_{j=1}^d \tilde p_j^{\eta_j(z) - 1},
  \qquad \tilde p \in \Simplex.
\]
Writing this density as
\begin{equation}
    f(\tilde p; \eta(z))=\exp\left(-\log B(\eta(z))+\sum_{j=1}^d(\eta_j(z)-1)\log \tilde p_j\right),
\end{equation}
we observe that this form matches Definition~\ref{def:exp_family} with $x=\tilde p$ and $\vartheta_i=\eta_i(z)-1$ for $i\in[d]$, along with 
\begin{align}
    \phi_j(\tilde p)&=\log\tilde p_j\\
    % \rho_j(\eta(z))&=\eta_j(z)-1\\
    A(\eta(z))&=\log B(\eta(z))\\
    h(\tilde p)&=1.
\end{align}
Therefore the exponential family identities in~\eqref{eq:exp_id_1} and~\eqref{eq:exp_id_2} give
\begin{align}
    \E_{\tilde p}[\log \tilde p_j]&= \frac{\partial}{\partial\eta_j}\log B(\eta(z))\\
    \textnormal{Cov}(\log \tilde p_i,\log \tilde p_j)&=\frac{\partial^2 }{\partial\eta_i \partial \eta_j}\log B(\eta(z)).
\end{align}\
We then find
\begin{align}
    \frac{\partial}{\partial\eta_j}\log B(\eta(z)) &=\frac{\partial}{\partial\eta_j}\left[ \sum_{k=1}^d\log \Gamma(\eta_k(z))-\log\Gamma(\tau)\right]\\
    &=\psi(\eta_j(z))-\psi(\tau),
\end{align}
where~$\tau$ is from~\eqref{eq:etayandtau}. 
Therefore, $\E_{\tilde p}[\log \tilde p_j]=\psi(\eta_j(z))-\psi(\tau)=\psi\bigl(r s_y(z) + \alpha\bigr)-\psi(r + d\alpha)$
and $F(z)=\E_{\tilde p}[-\log \tilde p_j]=\psi(r + d\alpha) -\psi\bigl(r s_y(z) + \alpha\bigr)$, 
which is exactly~\eqref{eq:F-closed-main}.

Computing the second derivative, we find 
\begin{align}
    \frac{\partial^2 }{\partial\eta_i \partial \eta_j}\log B(\eta(z))&=\frac{\partial }{\partial\eta_i}\left[\psi(\eta_j(z))-\psi(\tau)\right]\\
    &=\delta_{ij}\psi'(\eta_j(z))-\psi'(\tau).
\end{align}
Evaluating this expression with $i=j=y$ gives $\textnormal{Cov}(\log\tilde p_y, \log\tilde p_y)=\psi'(r s_y(z) + \alpha)-\psi'(r+d\alpha).$
Since $\mathrm{Var}_{\tilde p}\bigl[\loss(\tilde p(z), y)\bigr]=\mathrm{Var}_{\tilde p}\bigl[-\log \tilde p_y\bigr]=\mathrm{Var}_{\tilde p}\bigl[\log \tilde p_y\bigr]=\textnormal{Cov}(\log\tilde p_y, \log\tilde p_y),$
we have $\mathrm{Var}_{\tilde p}\bigl[\loss(\tilde p(z), y)\bigr]=\psi'(r s_y(z) + \alpha)-\psi'(r+d\alpha)$, which is exactly~\eqref{eq:prove_me_2}.

Before proving the result for $\E_{\tilde p}\left[\frac{\partial}{\partial z_j}\loss(\tilde p(z),y)\right]$, 
we first compute $\frac{\partial}{\partial z_j}\E_{\tilde p}\left[\loss(\tilde p(z),y)\right]$. 
Let $\loss_{\mathrm{CE}}(z) = -\log s_y(z)$.
%and let $s_k(z) = \frac{\exp(z_k)}{\sum_{\ell=1}^d e^{z_\ell}}$.
Then the partial derivatives with respect to $z_j$ are given by
\begin{align}
    \frac{\partial s_y(z)}{\partial z_j} &= s_y(z)(\delta_{yj} - s_j(z)) \label{eq:derivative1} \\
    \frac{\partial \loss_{\mathrm{CE}}(z)}{\partial z_j} \label{eq:derivative2}
  &= s_j(z) - \delta_{yj}.
\end{align}
%\mh{In the two lines above, aren't~$p_y$ and~$p_j$ functions of~$z$?} \calvin{yes}
Differentiating~\eqref{eq:prove_me_1} gives
\begin{align}
        \frac{\partial}{\partial z_j}\E_{\tilde p}\left[\loss(\tilde p(z),y)\right] &= \frac{\partial}{\partial z_j}\left[\psi(r + d\alpha) - \psi\bigl(r s_y(z) + \alpha\bigr)\right]\\
        &=-\psi'\bigl(r s_y(z) + \alpha\bigr)\cdot r\cdot \frac{\partial s_y(z)}{\partial z_j}\\
        &=-r\psi'(rs_y(z)+\alpha)\cdot s_y(z)(\delta_{yj}-s_j(z))\\
        &=rs_y(z)\psi'(rs_y(z)+\alpha)\cdot (s_j(z)-\delta_{yj})\\
        &=\kappa(z;r,\alpha)\frac{\partial \loss_{\mathrm{CE}}(z)}{\partial z_j}, \label{eq:derivative3}
\end{align}
where the third line uses~\eqref{eq:derivative1} and the fifth line uses~\eqref{eq:derivative2}.

Finally, we prove
\begin{equation}
    \E_{\tilde p}\left[\frac{\partial}{\partial z_j}\loss(\tilde p(z),y)\right]
    = \kappa\bigl(z; r, \alpha\bigr)\frac{\partial \loss_{\mathrm{CE}}}{\partial z_j}(z). \label{eq:dct_me}
\end{equation}

Since we have already shown in~\eqref{eq:derivative3} 
that $\frac{\partial}{\partial z_j}\E_{\tilde p}\left[\loss(\tilde p(z),y)\right]=\kappa(z;r,\alpha)\frac{\partial \loss_{\mathrm{CE}}}{\partial z_j}$, the equality in~\eqref{eq:dct_me} follows directly from an application of the Dominated Convergence Theorem, i.e.,
\begin{equation}
    \E_{\tilde p}\left[\frac{\partial}{\partial z_j}\loss(\tilde p(z),y)\right]=\frac{\partial}{\partial z_j} \E_{\tilde{p}}\left[\loss(\tilde p(z),y)\right]=\kappa(z;r,\alpha)\frac{\partial \loss_{\mathrm{CE}}}{\partial z_j},\label{eq:dct_id}
\end{equation}
and the remainder of the proof justifies the application of the Dominated Convergence Theorem~\cite[Theorem 1.8]{lieb2001analysis}.
To emphasize the density's dependence on $z$ we abuse notation and denote the  Dirichlet density by $f_z(u)=\frac{1}{B(\eta(z))}\prod_{j=1}^d u_i^{\eta_j(z)-1}$. 
We have $\E[-\log \tilde p_y]=\int_{S^{d-1}}-\log u_y f(u; \eta(z)) du$. 
To move $\frac{\partial}{\partial z_j}$ inside the integral with the Dominated Convergence Theorem, we must show that there exists an integrable function
$Z: \Simplex \to \mathbb{R}$
such that
\begin{equation}
    \left| \frac{\partial}{\partial z_j}\left[ -\log u_y f(u; \eta(z))\right]\right|\leq Z(u)
\end{equation}
for all $u\in S^{d-1}$. 
We first have
\begin{align}
    \frac{\partial}{\partial z_j}\left[ f(u; \eta(z))\right]&=f(u; \eta(z)) \frac{\partial}{\partial z_j}\log f(u; \eta(z))\\
    &=f(u; \eta(z))\left[-\frac{\partial}{\partial z_j}\log B(\eta(z))+\sum_{k=1}^d\frac{\partial\eta_k (z)}{\partial z_j}\log u_k\right]. \label{eq:derivative4}     
\end{align}
Using $\log B(\eta(z))=\sum_{k=1}^d\log \Gamma(\eta_k(z))-\log \Gamma(\tau)$
then gives $\frac{\partial}{\partial \eta_k}\log B(\eta(z))=\psi(\eta_k(z))-\psi(\tau)$, 
and applying the chain rule gives
\begin{equation}
    \frac{\partial}{\partial z_j}\log B(\eta(z))=\sum_{k=1}^d \frac{\partial \eta_k(z)}{\partial z_j}\big(\psi(\eta_k(z))-\psi(\tau)\big)
\end{equation}
and
\begin{equation}
    \frac{\partial}{\partial z_j}\left[ f(u; \eta(z))\right]=f(u; \eta(z))\sum_{k=1}^d \frac{\partial \eta_k(z)}{\partial z_j}\big(\log u_k-\psi(\eta_k(z))+\psi(\tau)\big).
\end{equation}
Therefore 
\begin{equation}
    \left| \frac{\partial}{\partial z_j}\left[ -\log u_y f(u; \eta(z))\right]\right|\leq|\log u_y|f(u; \eta(z))\sum_{k=1}^d \left|\frac{\partial \eta_k(z)}{\partial z_j}\right|\Big(|\log u_k|+|\psi(\eta_k(z))|+|\psi(\tau)|\Big).\label{eq:pre_bound}
\end{equation}
We now establish that the terms that depend on $z$ are bounded:
\begin{enumerate}
\item The density $f(u; \eta(z))$ is bounded:
\begin{enumerate}
    \item Since $\eta_k(z)\in[\alpha,\alpha+r]$ and $1/B(\cdot)$ is continuous     
    on $(0,\infty)$, there exists a constant $C_1<\infty$ such that $\left| \frac{1}{B(\eta(z))}\right|\leq C_1$.
    \item Because $u_k\in(0,1]$ and $\eta_k(z)\geq \alpha$, we have $u_k^{\eta_k(z)-1}\leq u_k^{\alpha-1}$. 
\end{enumerate}

    \item We have $\frac{\partial \eta_k(z)}{\partial z_j}=r\frac{\partial s_k(z)}{\partial z_j}$, 
    where $\frac{\partial s_k(z)}{\partial z_j}$ is the derivative of the softmax function, which is bounded, and there exists a constant $C_2$ such that $\left|\frac{\partial \eta_k(z)}{\partial z_j}\right|\leq C_2$. 
    
    \item Since $\eta_k(z)\in[\alpha,\alpha+r]$ and $\psi$ is continuous on $(0,\infty)$, 
    there exists a constant $C_3$ such that $\left| \psi(\eta_k(z))\right|\leq C_3$. 
     
    \item The value of $\psi(\tau)$ is constant because $\tau = r + d\alpha$. 
    
\end{enumerate}
Combining items~1-4 in~\eqref{eq:pre_bound} gives the bound
%\begin{equation}
%    \left| \frac{\partial}{\partial z_j}\left[ -\log u_y f(u; \eta(z))\right]\right|\leq C|\log u_y|\left( 1+\sum_{i=1}^d|\log u_i|\right)\prod_{i=1}^d  u_i^{\alpha-1}.
%\end{equation}
\begin{align}
    \left| \frac{\partial}{\partial z_j}\left[ -\log u_y f(u; \eta(z))\right]\right|&\leq C_1 C_2|\log u_y|\left( d(C_3+|\psi(\tau)|)+\sum_{j=1}^d|\log u_j|\right)\prod_{i=1}^d  u_i^{\alpha-1}\\
    &\leq C|\log u_y|\left( 1+\sum_{j=1}^d|\log u_j|\right)\prod_{i=1}^d  u_i^{\alpha-1},
\end{align}
where $C=C_1 C_2\Big(d\big(C_3+|\psi(\tau)|\big)+1\Big)$.

Now we show that this dominating function has a finite integral over the simplex.
We use Young's inequality, namely
\begin{equation}
    ab\leq\frac{a^2+b^2}{2},
\end{equation}
with $a=|\log u_y|$ and $b=1+\sum_{j=1}^d|\log u_j|$ to find
\begin{multline}
    C|\log u_y|\left( 1+\sum_{j=1}^d|\log u_j|\right)\prod_{i=1}^d  u_i^{\alpha-1}
    \leq \frac{C}{2} \left(| \log u_y|^2 +\left(1+\sum_{j=1}^d|\log u_j |\right)^2\right)\prod_{i=1}^d  u_i^{\alpha-1}\\
    \leq\frac{C}{2} \left(| \log u_y|^2 +(d+1)\left(1+\sum_{j=1}^d|\log u_j |^2\right)\right)\prod_{i=1}^d  u_i^{\alpha-1}, \label{eq:must_be_int}
\end{multline}
where the second inequality applies Cauchy-Schwarz.
For the function on the right-hand side of~\eqref{eq:must_be_int}
to be integrable, it suffices to show that
\begin{equation}
    \int_{S^{d-1}}|\log u_j|^2 \prod_{i=1}^d  u_i^{\alpha-1}du<\infty\label{eq:simp_int} 
\end{equation}
for all $j\in[d].$
In the scalar case, i.e., for $u\in(0,1]$ rather than $u\in S^{d-1},$
the corresponding integrals are finite because 
\begin{equation}
\int_0^1u^{\alpha-1}du = \frac{1}{\alpha}<\infty
\end{equation}
and
\begin{equation}
    \int_{0}^1 |\log u|^2u^{\alpha-1}du = \int_{0}^1 (-\log u)^2u^{\alpha-1}du=\frac{2}{\alpha^3}<\infty.
\end{equation}
For each $j\in[d]$, the simplex integral in~\eqref{eq:simp_int} has the same form
as one of the scalar integrals above.
Thus, the dominating function in~\eqref{eq:must_be_int} is integrable and the Dominated Convergence Theorem applies.

\subsection{Asymptotic Analysis and Empirical Results for Section~\ref{sec:analysis}} \label{ss:asymptotics}
We perform a brief asymptotic analysis to characterize how $\gamma_{\mathrm{eff}}$ 
in~\eqref{eq:JEFFISHERE}
varies with $\alpha.$
First, we account for the fact that $r^*$ is 
constrained 
by the condition~$h(r) \leq {\epsilon}$ from~\eqref{eq:optim}.
We approximate $r^*$ as a linear function of $\alpha$ of the form $c(\epsilon)\alpha.$
Table~\ref{tab:r-linear} shows that at $\epsilon=1,$ 
the behavior of 
$r^*$ is roughly linear in $\alpha$ with $c(1)\approx 0.08.$
 
Substituting the linear approximation into 
\eqref{eq:var-closed} and \eqref{eq:kappa-def-main}
and using $\psi'(x) \approx x^{-2}$ as $x \to 0$ and
$\psi'(x) \approx x^{-1}$ as $x \to \infty$ yields the two regimes
\begin{equation}
  \kappa(z; r, \alpha) \approx
  \begin{cases}
    \alpha^{-1}, & \alpha \ll 1 \\[0.2em]
    \text{const},  & \alpha \gg 1,
  \end{cases}
  \qquad
  \mathrm{Var}_{\tilde p}\bigl[\loss(z; \tilde p)\bigr] \approx
  \begin{cases}
    \alpha^{-2}, & \alpha \ll 1 \\[0.2em]
    \alpha^{-1}, & \alpha \gg 1.
  \end{cases}
  \label{eq:asymptotics}
\end{equation}

Both the scaling factor $\kappa(z; r, \alpha)$ and the per-step variance $\mathrm{Var}_{\tilde p}\bigl[\loss(z; \tilde p)\bigr]$ shrink 
as $\alpha$ grows,
but the variance shrinks faster.
Thus, as $\alpha$ increases, Algorithm~\ref{alg:FDP-SGD} takes smaller steps on average but also experiences substantially less privacy-induced variance.
This behavior makes runs with larger $\alpha$ less noisy 
and allows for stability at larger learning rates.

\begin{table}[t]
\centering

\begin{tabular}{ccc}
\toprule
$\alpha$ & $r^*$ & $r^*/\alpha$ \\
\midrule
0.1 & 0.0087 & 0.0870 \\
0.3 & 0.0260 & 0.0867 \\
0.5 & 0.0433 & 0.0866 \\
1 & 0.0864 & 0.0864 \\
2 & 0.1580 & 0.0790 \\
3 & 0.2349 & 0.0783 \\
\bottomrule
\end{tabular}
\caption{Values of $r^*$ found with~\eqref{eq:optim} for CIFAR10/ResNet-18, 
with $\epsilon = 1$,
$\delta = 10^{-5}$, a batch size of $250$, and $T = 20000$ steps.}
\label{tab:r-linear}
\end{table}

\subsection{Learning Rate Values for Figure~\ref{fig:pick_alpha}}
\label{apdnx:lr_values}
Here we provide the learning rate values tested for each $\alpha$ value in Figure~\ref{fig:pick_alpha}.
For $\alpha \in \{0.1, 0.2, 0.3, 0.5\}$ we use $11$ values, namely
$\gamma \in \{0.01, 0.025, 0.05, 0.075, 0.1, 0.2, 0.3, 0.4, 0.5, \\ 0.75, 1\}$. 
For $\alpha \in \{1, 2, 3\}$ we use $13$ values, namely
$\gamma \in \{0.05, 0.075, 0.1, 0.15, 0.2, 0.3, 0.5, 1, \\ 1.5, 2, 3, 5, 10\}$.
For $\alpha \in \{5, 10\}$ we use $11$ values, namely
$\gamma \in \{0.1, 0.2, 0.5, 0.75, 1, 2, 3, 4, 5, 6, \\ 7\}$.

\section{Implementation Details}
\label{appendix:implementation}
\subsection{Data Sets} We evaluate our framework on the following image classification data sets:
\begin{itemize}
    \item CIFAR10~\citep{krizhevsky2009learning} provides $32 \times 32$ color images categorized into $10$ distinct classes. The data set is partitioned into $50,000$ images for training and $10,000$ for testing.
    \item MNIST~\citep{lecun1998gradient} consists of $28 \times 28$ grayscale images of handwritten digits divided into $10$ classes. The standard benchmark includes $60,000$ training samples and $10,000$ testing samples.
    \item DermaMNIST~\citep{medmnistv1, medmnistv2}, part of the broader MedMNIST collection, focuses on classifying dermoscopic images into $7$ different types of skin lesions. The images are standardized to $28 \times 28$ pixels, with the data set split into $7,007$ training, $993$ validation, and $2,005$ test examples.
    \item FashionMNIST~\citep{xiao2017fashion} consists of Zalando's clothing articles. It shares the same $10$-class structure, 
    $28 \times 28$ grayscale format, and $60,000$/$10,000$ train-test split as MNIST. 
    \item SVHN~\citep{netzer2011reading} (Street View House Numbers) contains $32 \times 32$ RGB images of cropped digits captured from real-world street signs. It covers $10$ classes and provides $73,257$ training images alongside $26,032$ testing images in its core data set.
\end{itemize}

\subsection{Dirichlet Parameters}
\label{appendix:details}
In Table~\ref{tab:r_values}, we report the Dirichlet parameters $r^*$ 
and $\alpha$ for each data set and each privacy budget when using $\delta=10^{-5}$.
Each value of~$r^*$ is computed using~\eqref{eq:optim}.

\begin{table}[t]
  \centering
  \begin{tabular}{@{}l *{9}{c}@{}}
    \toprule
    Data set & $\alpha$ & $\epsilon=0.5$ & $\epsilon=1$ & $\epsilon=2$ & $\epsilon=3$ & $\epsilon=4$ & $\epsilon=6$ & $\epsilon=8$\\
    \midrule
    CIFAR10    & $3.0$ & --- & $0.235$ & $0.404$ & $0.559$ & $0.679$ & $0.825$ & $0.903$  \\
    MNIST       & $3.0$ & $0.136$ & $0.277$ & $0.472$ & --- & $0.711$ & --- & ---   \\
    DermaMNIST  & $3.0$ & --- & $0.137$ & --- & --- & $0.645$ & --- & $0.856_{(\epsilon=7.42)}$  \\
    SVHN        & $3.0$ & --- & $0.284$ & $0.475$ & $0.575$ & $0.717$ & $0.930$ & $0.964$   \\
    FashionMNIST& $3.0$ & $0.136$ & $0.277$ & $0.472$ & --- & $0.711$ & --- & ---   \\
    \bottomrule
  \end{tabular}
  \caption{Dirichlet parameters $r^*$ and $\alpha$ for each data set and each value of~$\epsilon$ when $\delta=10^{-5}$.}
  \label{tab:r_values}
\end{table}

\section{Additional Results on DermaMNIST}
\label{appendix:Derma}
We follow~\citet{holzl2022bridging} and report validation accuracy in Table~\ref{tab:Derma_val}. Our work outperforms the DP-SGD baseline~\citep{tang2023differentially} and~\cite{holzl2022bridging} across all evaluated
values of~$\epsilon$.

\begin{table}[b]
  \centering
  \small
  \begin{tabular}{@{}l c *{4}{c}@{}}
    \toprule
    Method & Model & $\epsilon=1$ & $\epsilon=4$ & $\epsilon=7.42$ & $\epsilon=\infty$ \\
    \midrule
    DP-SGD~\citep{tang2023differentially}    & ResNet-9 & $69.00$ & $71.78$ & $74.08$ & $77.27$ \\
    \citet{holzl2022bridging} & ResNet-9 & ---     & ---     & $74.17$ & $77.84$ \\
    \midrule
    \multirow{2}{*}{This work} & \multirow{2}{*}{ResNet-9} & $74.98$    & $76.27$    & $76.49$    & $79.26$ \\
                          &                           & $\pm 2.04$ & $\pm 1.22$ & $\pm 1.45$ &      \\
    \bottomrule
  \end{tabular}
  \caption{Validation accuracy (\%) on DermaMNIST under $\delta=10^{-5}$ across 5 independent runs for each $\epsilon$.}
  \label{tab:Derma_val}
\end{table}

\section{Experiments with Different Sampling Rates}
\label{appendix:diff_sample_rates}
We use the SGD optimizer with learning rate $\gamma = 0.1$, $T=20000$ steps, no momentum, and with random crop and horizontal flip on CIFAR10 with ResNet-18 as a baseline.
We then investigate the effects of different sampling rates under $(4,10^{-5})$-differential privacy with Dirichlet parameter $\alpha=3.0$.
As shown in Table~\ref{tab:diff_sample}, we observe that smaller sampling rates improves the accuracy of private training with Algorithm~\ref{alg:FDP-SGD}, while large sampling rates such as $q=0.04$ lead to worse accuracy.
This degradation in accuracy 
is caused by privacy amplification effects, since when the sampling rate is large, the privacy-utility tradeoff benefits less from privacy amplification.
In practice, a small sampling rate is desired when training with Algorithm~\ref{alg:FDP-SGD}.

\begin{table}[t]
  \centering
  \small
  \begin{tabular}{@{}l *{5}{c}@{}}
    \toprule
    Method
      & Model 
      & $q=0.001$
      & $q=0.01$
      & $q=0.02$
      & $q=0.04$ \\
    \midrule
    \multirow{2}{*}{This work}
      & \multirow{2}{*}{ResNet-18}
      & $89.41$
      & $82.94$
      & $71.18$
      & $56.49$
      \\
      &
      & $\pm3.63$
      & $\pm2.73$
      & $\pm5.53$
      & $\pm3.51$ \\
    \bottomrule
  \end{tabular}
  \caption{Test accuracy (\%) of this work using different sampling rates on CIFAR10 across 5 independent runs for $(4.0,10^{-5})$-differential privacy.}
  \label{tab:diff_sample}
\end{table}

\section{Results on Different Optimizer}
\label{appendix:diff_opt}
We also test our private training framework using ResNet-18 on CIFAR10 with the Adam optimizer~\citep{kingma2014adam}
using learning rate  $\gamma=0.001$, sampling rate $q=0.005$, $T=20000$ steps, no momentum, $\delta=10^{-5}$, 
and $\alpha=3.0$.

\begin{table}[t]
  \centering
  \small
  \begin{tabular}{@{}l *{9}{c}@{}}
    \toprule
    Method
      & Model 
      & $\epsilon = 1$
      & $\epsilon = 2$
      & $\epsilon = 3$
      & $\epsilon = 4$
      & $\epsilon = 6$
      & $\epsilon = 8$
      & $\epsilon = \infty$ \\
    \midrule
    \multirow{2}{*}{This work}
      & \multirow{2}{*}{ResNet-18}
      & $86.78$
      & $87.50$
      & $87.99$
      & $88.27$
      & $88.57$
      & $88.62$
      & $94.58$
      \\
      &
      & $\pm0.66$
      & $\pm0.61$
      & $\pm0.25$
      & $\pm0.39$
      & $\pm0.57$
      & $\pm0.24$
      & \\
    \bottomrule
  \end{tabular}
  \caption{Test accuracy (\%) of Algorithm~\ref{alg:FDP-SGD} using the Adam optimizer on CIFAR10 across 5 independent runs for each $\epsilon$.}
  \label{tab:Adam}
\end{table}

% \begin{table}[!ht]
%   \centering
%   \caption{\hr{Test accuracy (\%) of our work using RMSprop optimizer on CIFAR10.}}
%   \label{tab:rmsprop}
%   \begin{tabular}{@{}l *{9}{c}@{}}
%     \toprule
%     Method
%       & Model 
%       & $\hat{\epsilon} = 1$
%       & $\hat{\epsilon} = 2$
%       & $\hat{\epsilon} = 3$
%       & $\hat{\epsilon} = 4$
%       & $\hat{\epsilon} = 6$
%       & $\hat{\epsilon} = 8$
%       & $\hat{\epsilon} = \infty$ \\
%     \midrule
%     \multirow{2}{*}{Ours}
%       & \multirow{2}{*}{ResNet-18}
%       & $85.08$
%       & $91.91$
%       & $92.39$
%       & $92.94$
%       & $93.37$
%       & $93.39$
%       & $94.58$
%       \\
%       &
%       & $5.43$
%       & $0.41$
%       & $0.34$
%       & $0.25$
%       & $0.22$
%       & $0.31$
%       & \\
%     \bottomrule
%   \end{tabular}
% \end{table}

Table~\ref{tab:Adam} shows the privacy-utility tradeoff for Algorithm~\ref{alg:FDP-SGD} using the Adam optimizer. 
The accuracy increases as $\epsilon$ increases, which follows the same trend as 
using the SGD optimizer. 
Notably, even at strongest privacy level of $\epsilon=1$, the accuracy is $86.78\%$.
These results illustrate 
that Algorithm~\ref{alg:FDP-SGD} is optimizer-agnostic in the sense that any optimizer can be used without harming privacy, while also maintaining high accuracy.

\vskip 0.2in
\bibliography{ref}

\end{document}